\documentclass{article}

\usepackage[round]{natbib}
\usepackage{amssymb}
\usepackage{mathrsfs}
\usepackage{bbm}

\usepackage{amsthm}
\usepackage{amsmath}
\usepackage{natbib}
\usepackage[colorlinks,citecolor=blue,urlcolor=blue,filecolor=blue,backref=page]{hyperref}
\usepackage{graphicx}

\usepackage[ruled]{algorithm2e}

\usepackage{booktabs}       

\newtheorem{definition}{Definition}

\newtheorem{lemma}{Lemma}
\newtheorem{theorem}{Theorem}

\DeclareMathOperator*{\minimize}{minimize}
\DeclareMathOperator*{\argmin}{arg\,min}

\DeclareMathOperator{\E}{\mathbb{E}}

\begin{document}

\title{Adaptive Covariate Acquisition for Minimizing Total Cost of Classification}

\author{Daniel Andrade (andrade@nec.com) \\
	\footnotesize{Central Research Laboratories, NEC} \\ \\
	%
	Yuzuru Okajima (y-okajima@nec.com) \\
	\footnotesize{Central Research Laboratories, NEC}}
	
\date{}

\maketitle

\begin{abstract}
In some applications, acquiring covariates comes at a cost which is not negligible.
For example in the medical domain, in order to classify whether a patient has diabetes or not, measuring glucose tolerance can be expensive. Assuming that the cost of each covariate, and the cost of misclassification can be specified by the user, our goal is to
minimize the (expected) total cost of classification, i.e. the cost of misclassification plus the cost of the acquired covariates. 
We formalize this optimization goal using the (conditional) Bayes risk and describe the optimal solution using a recursive procedure. Since the procedure is computationally infeasible, we consequently introduce two assumptions: (1) the optimal classifier can be represented by a generalized additive model, (2) the optimal sets of covariates are limited to a sequence of sets of increasing size. 
We show that under these two assumptions, a computationally efficient solution exists.
Furthermore, on several medical datasets, we show that the proposed method 
achieves in most situations the lowest total costs when compared to various previous methods.
Finally, we weaken the requirement on the user to specify all misclassification costs
by allowing the user to specify the minimally acceptable recall (target recall).
Our experiments confirm that the proposed method achieves the target recall while 
minimizing the false discovery rate and the covariate acquisition costs better than previous methods.
\end{abstract}

\section{Introduction}

In some applications, acquiring covariates comes at a cost which is not negligible.
For example, in the medical domain, in order to classify whether a patient has diabetes or not, measuring glucose tolerance can be expensive. On the other hand, glucose tolerance can be a good indicator for diabetes, i.e. increases our chance of predicting diabetes (or its absence) correctly. 

The example illustrates that in the medical domain we often have to strike a balance between classification accuracy and the cost of acquiring the covariates. 
A rational criteria to decide on the best trade-off is to minimize the \emph{expected total cost of classification}: expected cost of misclassification plus the cost of acquired covariates.

In the first part of this article, we formalize the optimization of the expected total cost of classification using the (conditional) Bayes risk and describe the optimal solution using a recursive procedure. 
However, it turns out that the procedure is computationally infeasible due to basically two factors:
(1) calculating the Bayes risk requires to estimate a high dimensional integral,
(2) the number of different covariate acquisition paths is exponential in the number of covariates.

As a consequence, we introduce two assumptions: (1) the optimal classifier can be represented by a generalized additive model (GAM), (2) the optimal sets of covariates are limited to a sequence of sets of increasing size. 
We show that under these two assumptions, a computationally efficient solution exists.

Our framework requires that the user can specify the cost of misclassification: the false positive cost (cost of wrongly classifying a healthy person as having diabetes), and false negative cost (cost of classifying a diabetes patient as healthy). However, we show that the requirement on the user to specify the false negative cost can be replaced by the specification of a lower bound on the recall. This is motivated by the medical domain where it is more common to specify the minimally acceptable recall (target recall), rather than specifying the false negative cost.

Our main contributions are as follows:

\begin{enumerate}
	\item We describe the optimal solution for minimizing the expected total cost of classification, which has not been clarified in previous works like \citep{dulac2012sequential} and \citep{shim2018joint}.
	\item We prove that for a GAM, the estimation of the (conditional) Bayes risk reduces to a one dimensional density estimation and integral which can be solved computationally efficiently.
	\item We propose an effective heuristic to estimate an optimal monotone increasing sequence of covariate sets by learning the regression coefficients of GAM with a group lasso penalty.
	\item We prove that our framework can be used to guarantee a user-specified recall level, like 95\% which is common in the medical domain.
	\item We show on four medical datasets that the proposed method can lead 
	to lower total cost of classification than the previous works in \citep{ji2007cost,dulac2012sequential,xu2012greedy,nan2017adaptive,shim2018joint}.
	Furthermore, evaluation under the requirement of a target recall shows that the proposed method achieves the target recall while minimizing the remaining costs and false discovery rate (FDR) better than previous works.
\end{enumerate}

This article extends our preliminary work in \citep{andrade2019efficient} by replacing the linear classifier with GAM, allowing the specification of a minimal recall, 
determining the sequence of covariate sets by the solution path of a group lasso penalized convex optimization problem, and additional experimental evaluations on two more medical datasets.

In the next section, we formalize the optimal decision procedure to achieve, in expectation, the lowest total cost of classification.
In Section \ref{sec:COS}, we introduce a (non-adaptive) method that minimizes an upper of the lowest achievable total cost, which we extend in Section \ref{sec:AdaCOS} to an adaptive method.
In Section \ref{sec:covariateSetSelection}, we explain two approximations for finding a sequence of monotone increasing covariate sets that is used by the proposed method.
In Section \ref{sec:recallGuarantee}, we show how the proposed framework can also be used for guaranteeing a target recall.
Extensive empirical evaluations of our proposed method and previous methods are provided in Section \ref{sec:experiments}. In Section \ref{sec:relatedWork}, we give a concise review of related work. Finally, in Section \ref{sec:conclusions}, we summarize our findings.

\section{A cost rational selection criteria} \label{sec:costRationalSelectionCriteria}

Let  $L := \{l_1, \ldots, l_c\}$ denote the set of class labels, and $c_{y, y^*}$ the cost of classifying a sample as class $y^*$, when the true label is $y$.
A decision procedure $\delta^*: \mathbb{R}^p \rightarrow L$ for which
\begin{align*} 
\forall \delta: \, \E_{\mathbf{x},y} [c_{y,\delta(\mathbf{x})}] \geq \E_{\mathbf{x}, y} [c_{y,\delta^*(\mathbf{x})}] \, 
\end{align*}
is called a Bayes procedure.
The following procedure $\delta^*$ is a Bayes procedure (for a proof see, for example, Theorem 6.7.1 in \cite{anderson2003introduction}):
\begin{align} 
\label{eq:bayesProcedureDefinition}
\begin{split}
\delta^*(\mathbf{x}) 
&= \argmin_{y^* \in L} \sum_{y \in L}  p(y | \mathbf{x}) \cdot c_{y, y^*} \\
&= \argmin_{y^* \in L} \E_y [c_{y, y^*} ] \, .
\end{split}
\end{align}
The expected misclassification cost of the Bayes procedure, i.e. $\E_{\mathbf{x}, y} [c_{y,\delta^*(\mathbf{x})}]$, is called the Bayes risk.

Let us denote by $V := \{1, \ldots, p\}$ the index set of covariates with $V \cap L = \emptyset$. 
We denote the Bayes procedure for classifying a sample based only on the covariates $S \subseteq V$ by $\delta^*_{S} : \mathbb{R}^{|S|} \rightarrow L$. That means
\begin{align} \label{eq:bayesProcedureConditionalDefinition}
\delta^*_S(\mathbf{x}_S) = \argmin_{y^* \in L} \sum_{y \in L} p(y | \mathbf{x}_S) \cdot c_{y, y^*} \, .
\end{align}
When it is clear from the context, we drop the index on $\delta^*_S$, and just write $\delta^*(\mathbf{x}_S)$ instead of $\delta^*_S(\mathbf{x}_S)$.\footnote{Remark about our notation: we denote by bold font a column vector, e.g. $\mathbf{x} \in \mathbb{R}^p$, and a column vector indexed by a set $A \subseteq V$ denotes the corresponding sub-vector, e.g. $\mathbf{x}_A \in \mathbb{R}^{|A|}$.}


\subsection{Optimal Procedure} \label{sec:optimalProcedure}

The classical definition of Bayes procedure does not consider the cost of covariate acquisition, and assumes that all covariates are acquired at once. 
Therefore, let us first formally extend the definition appropriately. 

We use the following definition of a decision procedure. 
\begin{definition}
	A function of the form
	\begin{align*} 
	\pi: \mathbb{R}^p \times 2^V \rightarrow L \cup V \, ,
	\end{align*}
	which fulfills, $\forall \mathbf{x} \in \mathbb{R}^p, S \subseteq V$:
	\begin{align}
	& \pi(\mathbf{x}, S)  = \pi(\mathbf{x} \odot \mathbf{1}_S, S) \, ,  \label{eq:DPcondition1} \\
	& \pi(\mathbf{x}, S) \in L \cup (V \setminus S) \, ,  \label{eq:DPcondition2}
	\end{align}
	is called a decision procedure.\footnote{$\odot$ denotes the Hadamard product, and $\mathbf{1}_S \in \mathbb{R}^p$ is the vector that is one in all positions indexed by $S$, and zero otherwise.}
\end{definition}
The condition in Equation \eqref{eq:DPcondition1} means that a decision procedure uses only the covariates that are indexed by $S$; 
the condition in Equation \eqref{eq:DPcondition2} means that a decision procedure cannot select a covariate that is already in $S$. 
In summary, the decision procedure $\pi(\mathbf{x}, S)$ either classifies the current sample, or selects a new covariate based on the observations $\mathbf{x}_S$.
To simplify the notation, we write $\pi(\mathbf{x}_S)$ instead of $\pi(\mathbf{x}, S)$. Furthermore, we denote the cost of acquiring covariate $i$ by $c_i$.

Given a sample $\mathbf{x}$ with class label $y$, we denote the loss of a decision procedure $\pi$ as $l( (\mathbf{x}, y), \pi)$. 
The loss can be computed recursively as follows. 
Let $l( (\mathbf{x}, y), \pi) := l( (\mathbf{x}, y), \pi, \emptyset)$, with 
\begin{align} \label{eq:LossDefinition}
&l( (\mathbf{x}, y), \pi, S) =  
& \begin{cases}
c_{y, \pi(\mathbf{x}_S)} &  \text{if } \pi(\mathbf{x}_S) \in L , \\ 
c_{\pi(\mathbf{x}_S)} + l( (\mathbf{x}, y), \pi, S \cup \{ \pi(\mathbf{x}_S) \} ) & \text{else. }
\end{cases}
\end{align}
If not stated otherwise, we assume that all costs are non-negative, i.e. $c_i \geq 0$, and $c_{y,y'} \geq 0$.

\begin{theorem} \label{thm:optDecisionProcedure}
	The decision procedure $\pi^*$ defined by
	\begin{align} 
	\label{eq:optDecisionProcedure}
	\begin{split}
	\pi^*(\mathbf{x}_S) = 
	\argmin_{i \in L \cup (V \setminus S)} 
	\begin{cases}
	\E_{y} [c_{y,i} | \mathbf{x}_S ]  & \text{if } i \in L ,  \\
	c_i + \E_{\mathbf{x}_{V \setminus S}, y} \big[l((\mathbf{x}, y), \pi^*, S \cup \{i\}) | \mathbf{x}_{S} \big] & \text{else.}  
	\end{cases} \\
	\end{split}
	\end{align}
	is a Bayes procedure. That means for any other decision procedure $\pi$ we have
	\begin{align*} 
	\E_{\mathbf{x}, y} [l( (\mathbf{x}, y), \pi^*) ] \leq \E_{\mathbf{x}, y} [l( (\mathbf{x}, y), \pi) ] \, .
	\end{align*}
\end{theorem}

The proof is given in the appendix. We note that, if the covariates are discrete, we can formulate the problem as a stationary Markov decision process (MDP) where every policy leads to a terminal state \citep{zubek2004pruning,bayer2004learning}. 
The Bayes procedure from Theorem \ref{thm:optDecisionProcedure} is then equivalent to the optimal policy defined by the Bellman updates with the discounting factor set to 1 \citep{russell2003artificial}.

For continuous covariates, implementing the exact decision procedure $\pi^*$ is, in general, intractable.
The reason is that in order to recursively evaluate the loss, we need to evaluate a sequence of interchanging minimizations and expectations.
Therefore, we propose two relaxations and corresponding methods named Cost-sensitive Covariate Selection (COS) and Adaptive Cost-sensitive Forward Selection (AdaCOS) in Section \ref{sec:COS} and \ref{sec:AdaCOS}, respectively.

\section{Cost-sensitive Covariate Selection (COS)} \label{sec:COS}

Our first relaxation is to pull-out all minimizations in the recursion of Equation \eqref{eq:optDecisionProcedure} which leads to the following upper bound:
\begin{align*} 
\E_{\mathbf{x}, y} [ l( (\mathbf{x}, y), \pi^*) ] \leq \min_{S \subseteq V} \Big( \E_{\mathbf{x}_S, y}[c_{y,\delta^*(\mathbf{x}_S)}] + \sum_{i \in S} c_i \Big)
\end{align*}

In the following, we denote this upper bound by $\mathcal{U}$.
However, directly trying to minimize $\mathcal{U}$ is still computationally difficult due to the exponential 
number of possible sets $S$. Note that this is similar to covariate selection in logistic classification like e.g. in \cite{tibshirani1996regression,o2009review}.
Two important differences are that, in general, costs associated with covariates can be different from each other, and that the Bayes risk needs to be evaluate for all possible subsets $S \subseteq V$. The situation is illustrated in Figure \ref{fig:evaluationTreeExample}.

We denote the method selecting the set $S^*$ that minimizes $\mathcal{U}$,
by Cost-sensitive Covariate Selection (COS). In order to approximately find the set $S^*$ we will use the methods described in Section \ref{sec:covariateSetSelection}.
One disadvantage of COS is that it always selects the same set of covariates $S^*$ for any sample, though for some samples less/more covariates might be sufficient/necessary for good classification accuracy.

\begin{figure*}[t]
	\centering
	\includegraphics[trim=0 230 230 0,clip=true,scale=0.60,page=1]{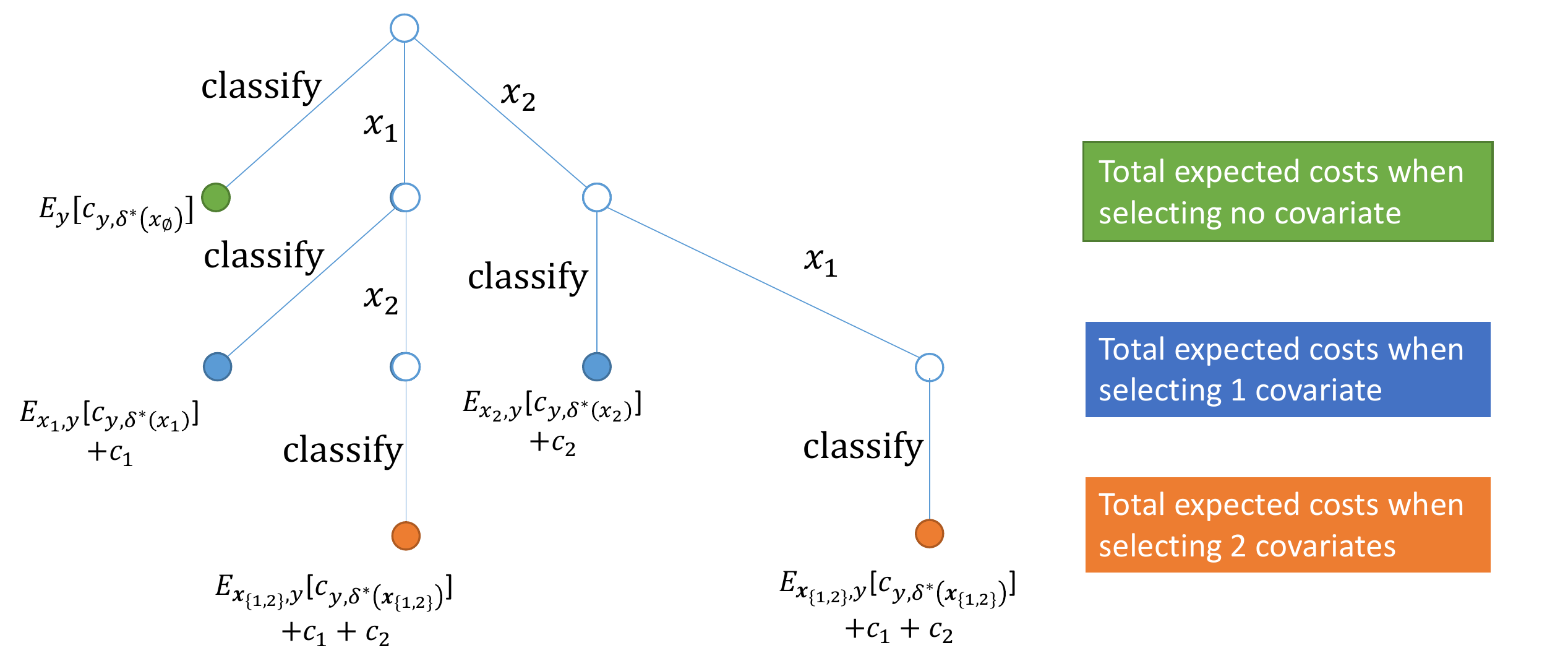}
	\caption{Shows an example of the expected total cost of classification at the beginning when no covariates have been acquired. Each edge represents one decision: either asking for the value of a covariate, or classifying the sample based on the observed covariates so far using the Bayes procedure. Each leaf shows the expected total cost of classification when using the covariates that were selected on the path from the root to the leaf. If we do not re-evaluate the expected cost after acquiring a new covariate, we will always select the same set of covariates, namely the method COS. \label{fig:evaluationTreeExample}}
\end{figure*}

\section{Adaptive Cost-sensitive Forward Selection} \label{sec:AdaCOS}

The method COS proposed in the previous section is not adaptive, i.e. 
it does not take into account the actually observed covariates in order to 
decide whether to proceed the acquisition of additional covariates, or whether to classify based on the covariates observed so far. 
However, as discussed in Section \ref{sec:optimalProcedure}, without any additional assumptions estimating the optimal procedure $\pi^*$ from Theorem \ref{thm:optDecisionProcedure} is computationally infeasible. 
We therefore introduce two assumptions:
\begin{enumerate}
	\item The optimal set S of acquired covariates belongs to 
	$\mathfrak{S} = \{S_1, S_2, \ldots S_q\}$, where $S_1 \subset S_2 \subset S_3 \ldots S_q \subseteq V$.
	\item The conditional class probability $p(y | \mathbf{x}_{S})$, for $S \in \mathfrak{S}$,   belongs to a logistic generalized additive model.
\end{enumerate}



Before we proceed, let us introduce our definition of future costs. Let $A \subseteq V$ and $S \subseteq V \setminus A$, then we define
\begin{align} \label{eq:definitionF}
F_{\mathbf{x}_{A}}(S) := 
\overbrace{\E_{\mathbf{x}_{S}, y} \big[ c_{y,\delta^*(\mathbf{x}_{A \cup S})}  | \mathbf{x}_{A} \big]}^{\text{(conditional) Bayes risk}} 
+ \overbrace{\sum_{i \in S} c_i}^{\text{covariate costs}} \, .
\end{align}
$F_{\mathbf{x}_{A}}(S)$ is the expected total additional cost of classification when we have already acquired the covariates $A$, and are planning to acquire additionally the covariates $S$ before classifying. 
In particular, the upper bound $\mathcal{U}$ can be expressed as 
$\min_{S \subseteq V} F_{\mathbf{x}_{\emptyset}}(S)$.

Our approximation of the Bayes procedure $\pi^*$ from Theorem \ref{thm:optDecisionProcedure} is given in Algorithm \ref{alg:AdaCOS}. First, we acquire all covariates indexed by $S_1$, and then check whether acquiring any additional covariates from $S_2 \setminus S_1, \ldots S_q \setminus S_1$ reduces the total cost of classification in expectation. If that is the case, we acquire the covariates in $S_2 \setminus S_1$, and proceed analogously. If the total cost of classification is not expected to decrease with more covariates, we stop and classify based on the covariates acquired so far. An example of the procedure is show in Figure \ref{fig:evaluationSequenceExample}.

\begin{algorithm}
	{\bf Input:} $S_1, \ldots, S_q$ \\
	$S_0 := \emptyset$ \\
	\For{$i \in \{1,\ldots, q-1\}$}{
		acquire $\mathbf{x}_{S_i \setminus S_{i-1}}$ \\
		\If{$\forall j \in \{i+1,.., q\}: F_{\mathbf{x}_{S_i}}(S_j \setminus S_i) \geq F_{\mathbf{x}_{S_i}}(\emptyset)$}{
			output class $\delta^*(\mathbf{x}_{S_i})$ \\
		}
	}
	\caption{Adaptive Cost-sensitive Forward Selection (AdaCOS) for classifying a test sample. \label{alg:AdaCOS}}
\end{algorithm}

The algorithm is adaptive in the sense that the expected future costs $F_{\mathbf{x}_{A}}(S)$ depend on the covariates $\mathbf{x}_{A}$ observed so far.
Therefore, we see that the effectiveness of the algorithm hinges on the non-trivial task of calculating $F_{\mathbf{x}_{A}}(S)$.

\begin{figure*}[t]
	\centering
	\includegraphics[trim=0 380 230 0,clip=true,scale=0.60,page=2]{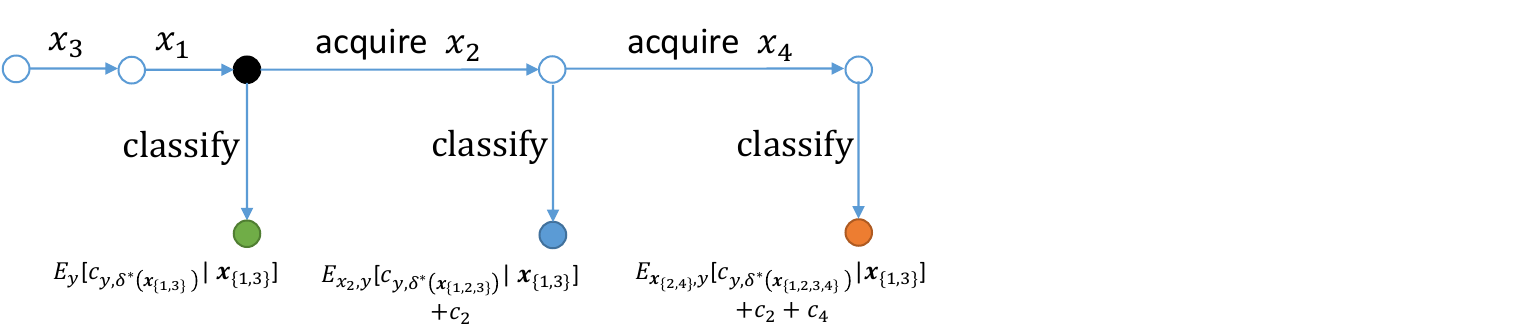}
	\caption{Shows an example where $S_1 = \{3\}$, $S_2 = \{3, 1\}$, $S_3 = \{3, 1, 2\}$, and 
		$S_4 = \{3, 1, 2, 4\}$. Assume we have acquired the covariates $S_2$. Then in order to decide whether to proceed acquisition, AdaCOS compares the current expected cost of misclassication marked in green with the expected future costs when additionally acquiring covariates $\{x_2\}$ or $\{x_2, x_4\}$, marked in blue and orange, respectively. \label{fig:evaluationSequenceExample}}
\end{figure*}

\subsection{Bayes Risk Estimation}
The main challenge in evaluating the future costs is to estimate the multi-dimensional integral in $\E_{\mathbf{x}_{S}, y} \big[ c_{y,\delta^*(\mathbf{x}_{A \cup S})}  | \mathbf{x}_{A} \big]$.
By assuming that the conditional class probability $p(y | \mathbf{x}_{A \cup S})$ can be modeled by a logistic generalized additive model, we will show that it is possible to reduce the multi-dimensional integral into a one-dimensional with an effective approximation.

The logistic generalized additive model is defined as follows.
Given the regression coefficients $\boldsymbol{\beta} = (\boldsymbol{\beta}_1, \boldsymbol{\beta}_2, \ldots, \boldsymbol{\beta}_p) \in \mathbb{R}^{s \cdot p}$, 
and intercept $\tau$, 
the conditional class probability $p(y = 1 | \mathbf{x}, \boldsymbol{\beta}, \tau)$ is modeled as 
\begin{align*}
p(y = 1 | \mathbf{x}_{A \cup S}, \boldsymbol{\beta}, \tau) = g\Big(\tau + \sum_{i \in A \cup S} \boldsymbol{\beta}_i^T f_i(x_i)\Big) \, ,
\end{align*}
where $g$ denotes the sigmoid function, and the non-linear transformations $f_i: \mathbb{R} \rightarrow \mathbb{R}^s$
are learned from the training data using penalized B splines, where $s$ is the number of splines \citep{hastie2009elements}.\footnote{For learning the B splines we make use of pyGAM available at \url{https://pygam.readthedocs.io/en/latest/}}


Furthermore, for simplicity, we assume that there are only two class labels $\{0, 1\}$, and $c_{0,0} = c_{1,1} = 0$.\footnote{Extension for allowing non-zero costs for $c_{y,y}$ is straight-forward, and omitted here.} We then have
%
%
%
\begin{align*}
& \E_{\mathbf{x}_{S}, y} \big[ c_{y,\delta^*(\mathbf{x}_{A \cup S})}  | \mathbf{x}_{A} \big] \\ 
&= \E_{\mathbf{x}_{S}} \big[  \sum_y c_{y,\delta^*(\mathbf{x}_{A \cup S})} p(y | \mathbf{x}_{A \cup S})  | \mathbf{x}_{A} \big] \\
&= \E_{\mathbf{x}_{S}} \big[  c_{0,\delta^*(\mathbf{x}_{A \cup S})} p(y=0 | \mathbf{x}_{A \cup S}) | \mathbf{x}_{A} \big]  \\
&\quad + \E_{\mathbf{x}_{S}} \big[ c_{1,\delta^*(\mathbf{x}_{A \cup S})} p(y=1 | \mathbf{x}_{A \cup S}) | \mathbf{x}_{A} \big] \, .
\end{align*}
Since
\begin{align*} 
\delta^*(\mathbf{x}_{A \cup S}) = \argmin [& p(y=1 | \mathbf{x}_{A \cup S}) \cdot c_{1, 0},  \\
& \,  p(y=0 | \mathbf{x}_{A \cup S}) \cdot c_{0, 1}] \, ,
\end{align*}
we have
\begin{align*} 
\delta^*(\mathbf{x}_{A \cup S}) = 1 &\Leftrightarrow \frac{p(y=1 | \mathbf{x}_{A \cup S}) \cdot c_{1, 0}}{p(y=0 | \mathbf{x}_{A \cup S}) \cdot c_{0, 1}} \geq 1 \\
&\Leftrightarrow \frac{g(\sum_{i \in A \cup S} \boldsymbol{\beta}_i^T f_i(x_i) + \tau) \cdot c_{1, 0}}{(1 - g(\sum_{i \in A \cup S} \boldsymbol{\beta}_i^T f_i(x_i) + \tau)) \cdot c_{0, 1}} \geq 1 \\
&\Leftrightarrow e^{\sum_{i \in A \cup S} \boldsymbol{\beta}_i^T f_i(x_i) + \tau} \geq \frac{c_{0, 1}}{c_{1, 0}} \\
&\Leftrightarrow \sum_{i \in A \cup S} \boldsymbol{\beta}_i^T f_i(x_i)  \geq \log (\frac{c_{0, 1}}{c_{1, 0}}) - \tau \\
&\Leftrightarrow \sum_{i \in S} \boldsymbol{\beta}_i^T f_i(x_i) \geq \log (\frac{c_{0, 1}}{c_{1, 0}}) - \tau - \sum_{i \in A} \boldsymbol{\beta}_i^T f_i(x_i) \\
&\Leftrightarrow z  \geq z^* \, ,
\end{align*}
where we defined $z := \sum_{i \in S} \boldsymbol{\beta}_i^T f_i(x_i)$, and $z^* := \log (\frac{c_{0, 1}}{c_{1, 0}}) - \tau - \sum_{i \in A} \boldsymbol{\beta}_i^T f_i(x_i)$.
We see that $\delta^*(\mathbf{x}_{A \cup S})$ depends only on $z$ (random variable) and $z^*$ (fixed). 
In the following, to simplify notation, let us denote by $h(z)$ the conditional distribution $p(z | \mathbf{x}_{A})$.
We thus have 
\begin{align*}
& \E_{\mathbf{x}_{S}} \big[ c_{1,\delta^*(\mathbf{x}_{A \cup S})} p(y=1 | \mathbf{x}_{A \cup S}) | \mathbf{x}_{A} \big] \\
&= \E_{\mathbf{x}_{S}} \big[ c_{1,\delta^*(\mathbf{x}_{A \cup S})} g(\sum_{i \in S} \boldsymbol{\beta}_i^T f_i(x_i) + \sum_{i \in A} \boldsymbol{\beta}_i^T f_i(x_i) + \tau) | \mathbf{x}_{A} \big] \\
&= \E_z \big[ c_{1,\delta^*(z, z^*)} g(z + \sum_{i \in A} \boldsymbol{\beta}_i^T f_i(x_i) + \tau) | \mathbf{x}_{A} \big] \\
&= \int c_{1,\delta^*(z, z^*)} g(z + \sum_{i \in A} \boldsymbol{\beta}_i^T f_i(x_i) + \tau) h(z) dz \\
&= \int_{- \infty}^{z^*} c_{1,0} g(z + \sum_{i \in A} \boldsymbol{\beta}_i^T f_i(x_i) + \tau) h(z) dz \\
&\quad+ \int_{z^*}^{\infty} c_{1,1} g(z + \sum_{i \in A} \boldsymbol{\beta}_i^T f_i(x_i) + \tau) h(z) dz \\
&= c_{1,0} \int_{- \infty}^{z^*}  g(z + \sum_{i \in A} \boldsymbol{\beta}_i^T f_i(x_i) + \tau) h(z) dz \, ,
\end{align*}
where we used that $c_{1,1} = 0$.
Analogously, we have 
\begin{align*}
& \E_{\mathbf{x}_{S}} \big[  c_{0,\delta^*(\mathbf{x}_{A \cup S})} p(y=0 | \mathbf{x}_{A \cup S}) | \mathbf{x}_{A} \big] \\
&= \E_{\mathbf{x}_{S}} \big[ c_{0,\delta^*(\mathbf{x}_{A \cup S})} (1- g(\sum_{i \in S} \boldsymbol{\beta}_i^T f_i(x_i) + \sum_{i \in A} \boldsymbol{\beta}_i^T f_i(x_i) + \tau)) | \mathbf{x}_{A} \big] \\
&= \E_z \big[ c_{0,\delta^*(z, z^*)}  (1- g(z + \sum_{i \in A} \boldsymbol{\beta}_i^T f_i(x_i) + \tau))  | \mathbf{x}_{A} \big] \\
&= c_{0,0} \int_{- \infty}^{z^*}  (1- g(z + \sum_{i \in A} \boldsymbol{\beta}_i^T f_i(x_i) + \tau)) h(z) dz  \\
&\quad + c_{0,1} \int_{z^*}^{\infty} (1 - g(z + \sum_{i \in A} \boldsymbol{\beta}_i^T f_i(x_i) + \tau)) h(z) dz \\
&= c_{0,1} \int_{z^*}^{\infty} (1 - g(z + \sum_{i \in A} \boldsymbol{\beta}_i^T f_i(x_i) + \tau)) h(z) dz \\
&= c_{0,1} \int_{z^*}^{\infty} h(z) dz  -  c_{0,1} \int_{z^*}^{\infty} g(z + \sum_{i \in A} \boldsymbol{\beta}_i^T f_i(x_i) + \tau) h(z) dz \, .
\end{align*}
Thus the remaining task is to evaluate the following integral
\begin{align} 
&\int_{a'}^{b'}  g(z + \sum_{i \in A} \boldsymbol{\beta}_i^T f_i(x_i) + \tau) h(z) dz \nonumber \\
&= \int_{a' + \sum_{i \in A} \boldsymbol{\beta}_i^T f_i(x_i) + \tau}^{b' + \sum_{i \in A} \boldsymbol{\beta}_i^T f_i(x_i) + \tau}  g(u) h(u - \sum_{i \in A} \boldsymbol{\beta}_i^T f_i(x_i) - \tau) du \, . \label{eq:dirtyNotationIntegral} 
\end{align}
We assume that $h(z) = p(z | \mathbf{x}_{A})$ can be well approximated by a normal distribution with mean $\mu_z$ and variance $\sigma^2$. We defer the explanation of how to estimate $\mu_z$ and $\sigma^2$ to Section \ref{sec:estimationConditionalProbZ}.

The integral in Equation \eqref{eq:dirtyNotationIntegral} has no analytic solution. One popular strategy is to approximate the sigmoid function $g$ by the cumulative distribution function of the standard normal distribution $\Phi$, as in Gaussian process classification \citep{rasmussen2006gaussian}. However, it turns out that this approximation is not applicable here, since $a'$ or $b'$ is a finite real number in our case. Instead, we use here the fact that the sigmoid function can be well approximated with only a few number of linear functions.
In order to facilitate notation, let us introduce the following constants:
\begin{align*}
a &:= a' + \sum_{i \in A} \boldsymbol{\beta}_i^T f_i(x_i) + \tau \, , \\
b &:= b' + \sum_{i \in A} \boldsymbol{\beta}_i^T f_i(x_i) + \tau \, , \\
\mu &:=  \mu_z + \sum_{i \in A} \boldsymbol{\beta}_i^T f_i(x_i) + \tau \, .
\end{align*}
Then we can write the integral in Equation \eqref{eq:dirtyNotationIntegral} as
\begin{align} \label{eq:niceNotationIntegral}
\int_{a}^{b}  g(u) \frac{1}{\sqrt{2 \pi \sigma^2}} e^{- \frac{1}{2\sigma^2} (u - \mu)^2} du \, .
\end{align}

Let us define the following piece-wise linear approximation of the sigmoid function:
\begin{align*}
g(u) \approx \sum_{t = 1}^{\xi + 2} \Big( 1_{[b_{t-1}, b_t]}(u) \big( m_t u + v_t \big) \Big) \, ,
\end{align*}
where for $1 \leq t \leq \xi + 1$, we set $b_{t} := -10 + \frac{20}{\xi} (t - 1) \, ,$ 
and for $1 \leq t \leq \xi$, we set
\begin{align*}
m_{t+1} :=  \frac{g(b_{t+1}) - g(b_t)}{b_{t+1} - b_t} \, , \qquad v_{t+1} := g(b_t) - m_{t+1} b_t \, , 
\end{align*}
and
\begin{align*}
& b_0 := - \infty \, , \, m_1 := 0 \, , \, v_1 := g(b_1)  \, , \\
& b_{\xi + 2} := + \infty \, , \, m_{\xi + 2} := 0 \, , \, v_{\xi + 2} := g(b_{\xi + 1}) \, ,
\end{align*}
and $\xi$ is the number of linear approximations, which is, for example, set to $40$. A comparison with the approximation $\Phi(\sqrt{\frac{\pi}{8}} u)$ is shown in Figure \ref{fig:sigmoidApproxComparison}.
That means for a relatively few number of linear approximations, we can achieve an approximation that is more accurate than the $\Phi$-approximation.
More importantly, as we show below, this allows for a tractable calculation of the integral in Equation \eqref{eq:niceNotationIntegral}, which is \emph{not} the case when using the $\Phi$-approximation.
Then we have
\begin{align*}
&\int_{a}^{b}  g(u) \frac{1}{\sqrt{2 \pi \sigma^2}} e^{- \frac{1}{2\sigma^2} (u - \mu)^2} du \\
&= \int_{a}^{b} \sum_{t = 1}^{\xi + 2} \Big( 1_{[b_{t-1}, b_t]}(u) \big( m_t u + v_t \big) \Big) \frac{1}{\sqrt{2 \pi \sigma^2}} e^{- \frac{1}{2\sigma^2} (u - \mu)^2} du \\
&= \sum_{t = 1}^{\xi + 2} m_t \int_{\max(a, b_{t-1})}^{\min(b, b_t)}  u \frac{1}{\sqrt{2 \pi \sigma^2}} e^{- \frac{1}{2\sigma^2} (u - \mu)^2} du  \\
&\quad+ v_t \Phi_{\max(a, b_{t-1})}^{\min(b, b_t)} \, ,
\end{align*}
where we define $\Phi_l^o := \int_l^o \frac{1}{\sqrt{2 \pi \sigma^2}} e^{- \frac{1}{2\sigma^2} (u - \mu)^2} du$, which can be well approximated with standard numerical libraries.
The remaining integral can also be expressed by $\Phi$ using the substitution $u - \mu := r$, we have 
\begin{align*}
&\int_{l}^{o}  u \frac{1}{\sqrt{2 \pi \sigma^2}} e^{- \frac{1}{2\sigma^2} (u - \mu)^2} du \\
&= \int_{l - \mu}^{o - \mu} r \frac{1}{\sqrt{2 \pi \sigma^2}} e^{- \frac{1}{2\sigma^2} r^2} dr + \mu \int_{l - \mu}^{o - \mu} \frac{1}{\sqrt{2 \pi \sigma^2}} e^{- \frac{1}{2\sigma^2} r^2} dr \\
&= \frac{\sigma}{\sqrt{2 \pi}} \big( e^{- \frac{1}{2\sigma^2} (l - \mu)^2} - e^{- \frac{1}{2\sigma^2} (o - \mu)^2} \big)  + \mu \Phi_{l}^{o} \, .
\end{align*}

\begin{figure}[t]
	\centering
	\includegraphics[scale=0.55,page=1]{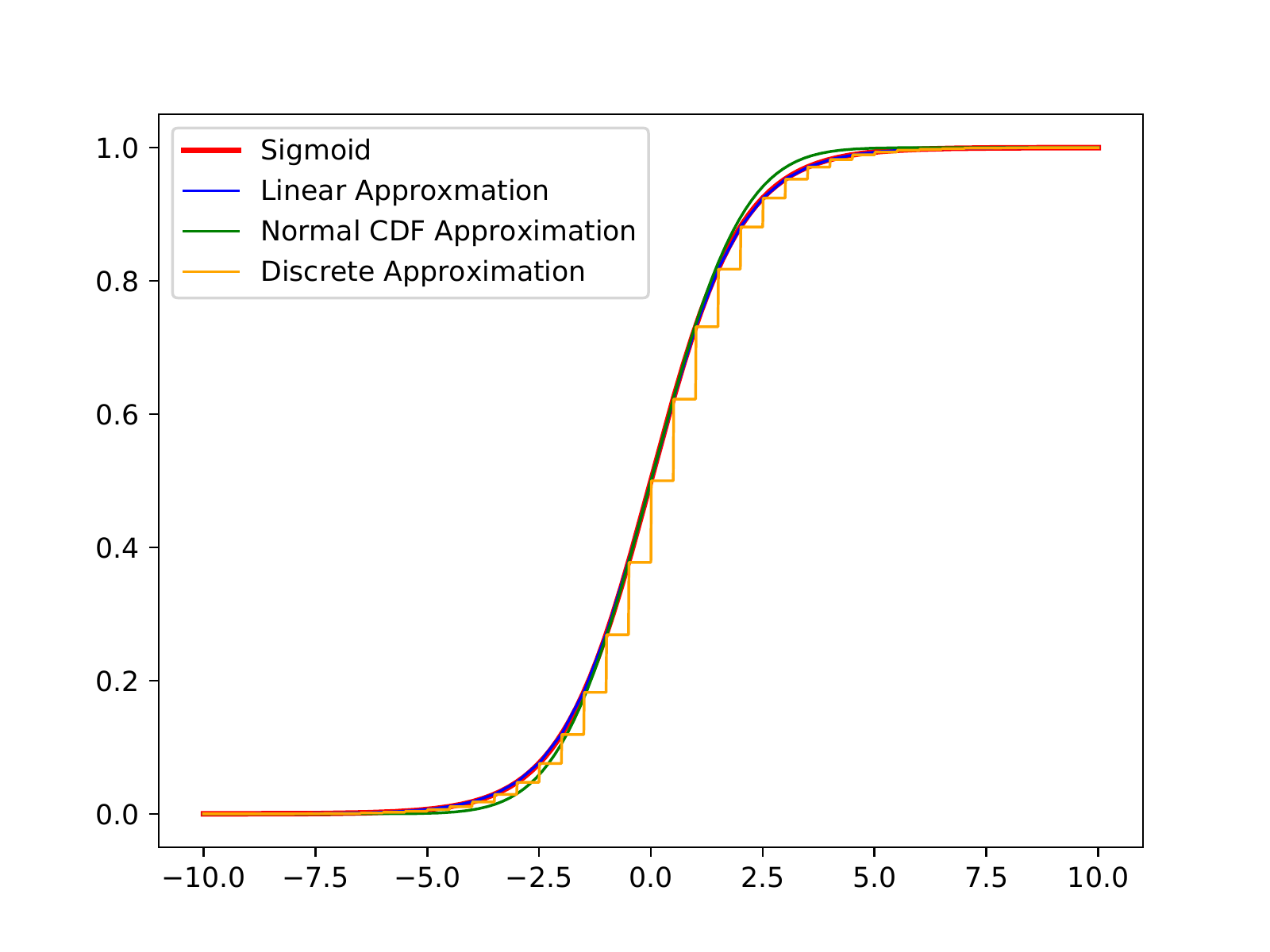}
	\caption{Comparison of Sigmoid function approximations. For the linear function approximation and the discrete bin approximation \citep{ji2007cost} we set $\xi = 40$. For the normal CDF approximation we use $\Phi(\sqrt{\frac{\pi}{8}} u)$.
		\label{fig:sigmoidApproxComparison}}
\end{figure}

\subsubsection{Estimation of $\mu_z$ and $\sigma^2$} \label{sec:estimationConditionalProbZ}
Recall that we assumed that $p(z | \mathbf{x}_{A})$ can be well approximated by a normal density with mean $\mu_z$ and variance $\sigma^2$. 
In order to estimate $\mu_z$ and $\sigma^2$, we propose to model $z$ given $\mathbf{x}_{A}$ as a regression problem with additive noise, where
$z$ is the response variable, and $\mathbf{x}_{A}$ are the  explanatory variables. 
In detail, for learning the regression model from the training data $\{\mathbf{x}^{(k)} \}_{k = 1}^{n}$, we prepare a 
collection of response and explanatory variable pairs of the form $\{ (z^{(k)}, \mathbf{x}_A^{(k)}) \}_{k = 1}^{n}$, where $z^{(k)} = \sum_{i \in S} \boldsymbol{\beta}_i^T f_i(x_i^{(k)})$.
We note that for training the regression model, we do not require the class label $y$. As a consequence, additional to the class-labeled training data, we could also exploit unlabeled training data (if available).\footnote{Though, for the medical datasets we consider in Section \ref{sec:experiments}, no unlabeled training data was available.}

For our experiments, we use a standard Bayesian linear regression model with a scaled inverse $\chi^2$ distribution prior on the noise variance \citep{gelman2013bayesian}.

\section{Cost-aware non-linear covariate selection} \label{sec:covariateSetSelection}

In the previous section, we assumed that the covariates are acquired in a specific sequence. 
In this section, we discuss two different approximation strategies for finding the optimal sequence of subsets $\mathfrak{S} = \{S_1, S_2, \ldots S_q\}$, where $S_1 \subset S_2 \subset S_3 \ldots S_q \subseteq V$, such that the expected total cost of classification tends to minimal for a set $S \in \mathfrak{S}$.

\subsection{Forward Selection} \label{sec:forwardSelection}

We suggest to set $q := p + 1$, and use greedy forward selection as outlined in Algorithm \ref{alg:subsetSelectionSearch}.

\begin{algorithm}
	{\bf Input:} $\{ (y^{(k)}, \mathbf{x}^{(k)}) \}_{k = 1}^{n}$ \\ 
	$S_1 := \emptyset$ \\
	\For{$i \in \{1,\ldots, p \}$}{
		$S_{i+1} := \argmin_{j \in V \setminus S_i} \E_{\mathbf{x}_{S_i}} [F_{\mathbf{x}_{S_i}} ( \{ j \} )] $ \\
	}
	\caption{Forward selection for finding subsets $S_1 \subset S_2, \ldots S_{p+1} \subseteq V$. 
		\label{alg:subsetSelectionSearch}}
\end{algorithm}


Note that from the definition in Equation \ref{eq:definitionF}, we have
\begin{align*} 
\E_{\mathbf{x}_{S}} [F_{\mathbf{x}_{S}} ( \{ j \} )]  
= \E_{\mathbf{x}_{S}} \big[ \E_{{x_j}, y} [ c_{y,\delta^*(\mathbf{x}_{S \cup \{j\}})}  | \mathbf{x}_{S} ] \big] + \sum_{i \in S} c_i \, .
\end{align*}

We estimate $\E_{\mathbf{x}_{S}} [F_{\mathbf{x}_{S}} ( \{ j \} )]$ using 10-fold cross-validation of the labeled training data $\{ (y^{(k)}, \mathbf{x}^{(k)}) \}_{k = 1}^{n}$. 
In particular, for one fold $\mathscr{A}_f \subseteq \{1, \ldots, n\}$, we have
\begin{align*} 
&\E_{\mathbf{x}_{S}} \big[ \E_{{x_j}, y} [ c_{y,\delta^*(\mathbf{x}_{S \cup \{j\}})}  | \mathbf{x}_{S} ]\big] \\
&= \E_{\mathbf{x}_{S \cup \{j\}}, y} [ c_{y,\delta^*(\mathbf{x}_{S \cup \{j\}})} ] \nonumber \\
&\approx \frac{1}{| \mathscr{A}_f |} \sum_{k \in \mathscr{A}_f}  c_{y^{(k)},\delta_f(\mathbf{x}_{S \cup \{j\}}^{(k)})}\, ,
\end{align*}
where the model for the conditional probability $p(y | \mathbf{x}_{S \cup \{j\}})$ used by 
$\delta_f$ is trained using the samples in $\{1, \ldots, n\} \setminus \mathscr{A}_f$.
The final estimate is acquired by averaging over all folds.

An advantage of the above forward-selection procedure is that it uses an unbiased estimate of $\E_{\mathbf{x}_{S}} [F_{\mathbf{x}_{S}} ( \{ j \} )]$, and assuming the variance is not too large, we can expect to find a good local minima.

However, there are several disadvantages.
First, if the variance of the estimator is high, we might get stuck in a bad local minima. 
Second, the forward-selection procedure is extremely computationally expensive, and, as a consequence, unfeasible if $p$ is large.
Finally, a more subtle disadvantage is that it requires the full specification of the misclassification costs, i.e. the specification of $c_{0,1}$ and $c_{1,0}$. As a consequence, it is not applicable when we are provided only with $c_{0,1}$, which we will discuss in Section \ref{sec:recallGuarantee}.

\subsection{Group Lasso Penalty} \label{sec:groupLasso}

Some of the disadvantages of the feed-forward selection method can be overcome by jointly training the model for the conditional probability $p(y | \mathbf{x})$ with a sparsity-enforcing penalty on the regression coefficients. Here, this is possible since we assume a generalized additive model for $p(y | \mathbf{x})$.

In particular, we propose to acquire the sets $S_1 \subset S_2 \subset S_3 \ldots S_q$ by using the search path of a penalized logistic loss function. 
In detail, for different values of $\lambda$, we solve the following convex optimization problem
\begin{align} \label{eq:nonLinearCovariateSelection}
\minimize_{\boldsymbol{\beta}, \tau} - \sum_{k = 1}^n \log p(y = y^{(k)} | \mathbf{x}^{(k)}, \boldsymbol{\beta}, \tau)  + \lambda \sum_{i = 1}^p c_i ||\boldsymbol{\beta}_i||_2 \, ,
\end{align}
where $\boldsymbol{\beta} = (\boldsymbol{\beta}_1, \boldsymbol{\beta}_2, \ldots, \boldsymbol{\beta}_p) \in \mathbb{R}^{s \cdot p}$.
The group lasso penalty ensures that the regression coefficients $\boldsymbol{\beta}_i$  are either all zero or all non-zero \citep{hastie2015statistical}. 
Note that in Equation \eqref{eq:nonLinearCovariateSelection} each group is scaled by $c_i$ which ensures that the regression coefficient of covariates with high cost are penalized more.\footnote{In order to make this type of penalty meaningful we ensure that each $f_i(x_i)$ has mean 0 and standard deviation 1.} As a consequence, in order to be included into the final model, covariates with high cost are required to lower the negative log-likelihood term more than covariates with low costs. 
By inspecting the search path for different values of $\lambda_1 > \lambda_2 > \ldots > \lambda_q$, we acquire the sets $S_1 \subset S_2 \subset S_3 \ldots S_q \subseteq V$.

the conditional class probability $p(y = 1 | \mathbf{x}, \boldsymbol{\beta}, \tau)$ is modeled as 
\begin{align}
p(y = 1 | \mathbf{x}, \boldsymbol{\beta}, \tau) = g(\tau + \sum_{i = 1}^{p} \boldsymbol{\beta}_i^T f_i(x_i)) \, .
\end{align}

The non-linear transformations $f_i: \mathbb{R} \rightarrow \mathbb{R}^s$, where $s$ is the number of splines, are learned from the training data using penalized B splines \citep{hastie2009elements}.


\section{Extension to Classification With Recall Guarantees} \label{sec:recallGuarantee}

So far, we assumed that both misclassification costs $c_{0, 1}$ and $c_{1, 0}$ are given. 
Arguably, the false positive cost $c_{0, 1}$ is relatively easy to specify. For example, in the medical domain, this might correspond to the price of a medicine which was unnecessarily prescribed to a healthy patient.

On the other hand, the specification of $c_{1, 0}$ is more difficult. For example, specifying the cost of a dead patient (that might have been rescued) is difficult.
Therefore, in the medical domain, it is more common to try to make a guarantee on the recall\footnote{Here we use the terminology from the machine learning literature, though, in the medical literature the term "sensitivity" is more common than "recall".}. In particular, it is common practice to require that the recall is $95\%$ \citep{kanao2009psa}.

In the following, we show how to estimate the false negative cost $c_{1, 0}$, given 
the false positive cost $c_{0,1}$ and the requirement that the recall is larger or equal to some value $r$.

In the following, we denote by $\mathbbm{1}_M$ the indicator function which is 1 if expression $M$ is true and otherwise 0.

Given a distribution over $(y, \mathbf{x})$ such that $\E[1_{y=1}] > 0$, the recall of a decision procedure $\delta$ is defined as:
\begin{align*}
R_\delta &:= \int p(\mathbf{x} | y = 1) \cdot \mathbbm{1}_{\delta(\mathbf{x}) = 1}  d \mathbf{x} \, .
\end{align*}

Assuming that $\delta$ is a Bayes procedure, we have 
\begin{align*}
\delta(\mathbf{x}) = 1 &\Leftrightarrow p(y = 1 | \mathbf{x}) \cdot c_{1, 0} \geq p(y = 0 | \mathbf{x}) \cdot c_{0, 1}  \\
&\Leftrightarrow p(y = 1 | \mathbf{x}) \cdot c_{1, 0} \geq (1 - p(y = 1 | \mathbf{x})) \cdot c_{0, 1} \\
&\Leftrightarrow p(y = 1 | \mathbf{x}) \geq \frac{c_{0, 1}}{c_{1, 0} + c_{0, 1}} \, .
\end{align*}
Setting 
\begin{align*}
t := \frac{c_{0, 1}}{c_{1, 0} + c_{0, 1}} \, ,
\end{align*}
we have that 
\begin{align*}
R_\delta &= \int p(\mathbf{x} | y = 1) \cdot \mathbbm{1}_{p(y = 1 | \mathbf{x}) \geq t}  d \mathbf{x} \, .
\end{align*}
In order to emphasize the dependence on $t$, we write in the following $R_t$ instead of $R_\delta$.
In particular, let $t^*$ be chosen such that 
\begin{align} \label{eq:recallEquation}
r = R_{t^*}
\end{align}
where $r$ is, for example, $0.95$. Then the implicitly defined cost $c_{1, 0}$ is given by
\begin{align} \label{eq:falseNegativeCostFromThreshold}
c_{1, 0} = \frac{1- t^*}{t^*} c_{0, 1} \, .
\end{align}
It remains to show how $t^*$ can be estimated.
In general, Equation \eqref{eq:recallEquation} does not have a solution (in terms of $t$). We therefore solve the following problem
\begin{align} \label{eq:guarantedRecalProblem}
& \max_t \, t \; , \, \text{subject to} \; \; r \leq R_t \, , 
\end{align}
which has always a solution (since $t = 0$ fulfills the constraint).
Since $p(\mathbf{x} | y = 1)$ is unknown, we use the empirical training data distribution to estimate $R_t$:
\begin{align} \label{eq:empericalRecallEstimate}
R_t &= \int p(\mathbf{x} | y = 1) \cdot \mathbbm{1}_{p(y = 1 | \mathbf{x}) \geq t}  d \mathbf{x} \nonumber \\ 
&\approx \frac{1}{n_1} \sum_{k : y^{(k)} = 1}  \mathbbm{1}_{p^{(-k)}(y = 1 | \mathbf{x}^{(k)}) \geq t} \, , 
\end{align}
where $n_1$ is the number of true samples (i.e. label $y = 1$), and $p^{(-k)}(y = 1 | \mathbf{x}^{(k)})$ is the estimate of $p(y = 1 | \mathbf{x}^{(k)})$ of the classifier that was trained without sample $k$. In practice, since this type of leave-one-out estimation is computationally expensive, we use instead 10-fold cross-validation.

Since the expression in Equation \eqref{eq:empericalRecallEstimate} 
is a monotone decreasing step function in $t$, we can easily solve the problem in \eqref{eq:guarantedRecalProblem} by sorting \\
$\big\{ p^{(-k)}(y = 1 | \mathbf{x}^{(k)}) \big\}_{k=1}^{n_1}$ in decreasing order. 

\subsection{Adaptive Covariate Acquisition With Recall Guarantees} \label{sec:recallGuaranteeAdaptive}

So far, we discussed how to estimate $c_{1,0}$ in the situation where only one classifier based on $p(y = 1 | \mathbf{x})$ is used. However, in the adaptive acquisition setting, the situation is more complicated since, in general, for different observed sets of variables, the conditional class probabilities are different. In particular, let $S \subset S' \subseteq V$, where $V$ is the set of all variables.
Then, in general, we have
\begin{align*}
p(y = 1 | \mathbf{x}_S) \neq p(y = 1 | \mathbf{x}_{S'}) \, ,
\end{align*}
which means that, in general, the optimal threshold $t^*$ which guarantees recall $\geq r$ is different for different sets of observed variables $S$.
Furthermore, in the adaptive setting, estimating the recall using Equation \eqref{eq:empericalRecallEstimate} is not valid anymore since the distribution of the samples, with label $y = 1$, and for which we select the variable set $S'$ is, in general, different from $p(\mathbf{x} | y = 1)$, i.e. $p(\mathbf{x} | y = 1, \text{acquired $S'$}) \neq p(\mathbf{x} | y = 1)$.

Nevertheless, we show in the following that it is possible to define the cost $c_{1,0}$ such that the recall requirement is fulfilled. 
First, let us introduce the following notations.
Let $S_1 \subset S_2 \subset S_3 \ldots S_q$, be the sets of variables that are considered for adaptive variable acquisition, i.e. first we acquire $S_1$, and then we decide whether to classify or whether we acquire additionally the variables in $S_2 \setminus S_1$, and so forth.
Moreover, let $\delta_{t, S}$ be the classifier based on the observed variables $S$ and using threshold $t$, i.e. 
\begin{align*}
\delta_{t, S}(\mathbf{x}_{S}) := 
\begin{cases}
1 & \text{if } p(y = 1 | \mathbf{x}_{S}) \geq t \, , \\
0 & \text{else.}
\end{cases}
\end{align*}
To simplify the notation, we write in the following $\delta_{t, S}$ short for $\delta_{t, S}(\mathbf{x}_{S})$.

Strict control of the recall can be achieved by requiring that 
\begin{align} \label{eq:allOneRequirement}
p(\delta_{t_{1}, S_{1}} = 1, \delta_{t_{2}, S_{2}} = 1, \ldots, \delta_{t_{q}, S_{q}} = 1 | y = 1) \geq r  \, .
\end{align}
This can be seen as follows. Assume that $y = 1$ and any classifier $\delta_{t_{j}, S_{j}}$ outputs label $0$, then an adversarial selection strategy will select this classifier. 
Otherwise, if all classifiers output label $1$, then even an adversarial selection strategy needs to select a classifier $\delta_{t_{j}, S_{j}}$ for which the output is $1$. By the requirement of Inequality \eqref{eq:allOneRequirement}, the latter will happen with probability of at least $r$.

If we require that all thresholds are the same, i.e. 
\begin{align*}
t = t_{1} = t_2 = \ldots = t_q,
\end{align*}
then we can proceed as before.
That means, based on Inequality \eqref{eq:allOneRequirement}, we first calculate $t^*$, and then specify the false negative cost $c_{1,0}$ using Equation \eqref{eq:falseNegativeCostFromThreshold}.
Analogously to before for checking Inequality \eqref{eq:allOneRequirement}, we use the empirical training data estimate:
\begin{align*}
&p(\delta_{t, S_{1}} = 1, \delta_{t, S_{2}} = 1, \ldots, \delta_{t, S_{q}} = 1 | y = 1) \\
&\approx 
\frac{1}{n_1} \sum_{k : y^{(k)} = 1}  \mathbbm{1}_{p^{(-k)}(y = 1 | \mathbf{x}_{S_1}^{(k)}) \geq t} \cdot  \mathbbm{1}_{p^{(-k)}(y = 1 | \mathbf{x}_{S_2}^{(k)}) \geq t} \ldots  \cdot \mathbbm{1}_{p^{(-k)}(y = 1 | \mathbf{x}_{S_q}^{(k)}) \geq t} \, .
\end{align*}

\section{Experiments} \label{sec:experiments}


We evaluate our proposed method on four real datasets from the medical domain
which are frequently used for cost-sensitive classification: 
Pima Diabetes dataset (p = 8, n = 768), the Wisconsin Breast Cancer dataset (p = 10, n = 683), Heart-disease dataset (p = 13, n = 303), and the PhysioNet dataset (p = 30, n = 12000). The first three datasets are all available at the UCI Machine Learning repository\footnote{\url{https://archive.ics.uci.edu/ml/index.html}}, the PhysioNet data is available at \url{https://archive.physionet.org/pn3/challenge/2012/}.

Note that the PhysioNet data \citep{goldberger2000physiobank} contains for each patient several health check measures like cholesterol, taken at different times during their stay in the intensive care unit.
As in \citep{shim2018joint}, for each patient we use the last recorded value of each attribute to predict death ($y = 1$) or survival ($y = 0$). After filtering attributes which are mostly missing, we acquire a data set with 12000 patients and 30 attributes.

For Diabetes and Heart-disease we use the covariate costs as defined in \citep{ji2007cost}, and \citep{turney1994cost}, respectively. For the other datasets, we set the covariate costs uniformly to one. 

Note that the Heart-disease and PhysioNet data contain missing values. For methods which cannot handle missing values (including our proposed methods) we assume that all covariates are jointly distributed according to a multivariate normal distribution, where the covariance matrix is estimated from all samples (including missing values) using the method from \cite{lounici2014high}. 

We compare the proposed method AdaCOS to the following methods:
\paragraph{COS} 
The proposed method but fixing the covariate set $S \in \{S_1, S_2,\ldots, S_q\}$ to the one which minimizes the total costs in expectation, i.e. 
\begin{align*}
\E_{\mathbf{x}_{S}, y} \big[  c_{y,\delta^*(\mathbf{x}_{S})} \big] + \sum_{i \in S} c_i \, ,
\end{align*}
which is estimated using 10-fold crossvalidation as in Section \ref{sec:forwardSelection}.\footnote{In case where the target recall is specified, we first estimate the false negative cost as in Section \ref{sec:recallGuaranteeAdaptive}, and then proceed as before.}
\paragraph{Full Model} 
The logistic generalized additive model which always acquires (and uses) all covariates.
\paragraph{Shim2018}
The cost-sensitive classification method based on deep reinforcement learning as proposed in \citep{shim2018joint}.\footnote{Available at \url{https://github.com/OpenXAIProject/Joint-AFA-Classification}}
\paragraph{GreedyMiser}
The cost-sensitive tree construction method proposed in \citep{xu2012greedy}.\footnote{Available at \url{http://kilian.cs.cornell.edu/code/code.html}}
\paragraph{AdaptGbrt}
The method proposed in \citep{nan2017adaptive}, which requires the specification of a high accuracy classifier for which we use the Full Model.\footnote{Available at \url{https://github.com/fnan/AdaptApprox}}

For all methods we estimate the hyperparameters with 10-fold crossvalidation, except where this is too computationally expensive: for Shim2018 we use the hold-out data split as in their provided implementation, for AdaptGbrt we use 5-fold crossvalidation. 

As evaluation measure, we use the average total cost of classification, defined as 
\begin{align*} 
\text{avg total cost} := \frac{1}{n_t} \sum_{k = 1}^{n_t} \big( c_{y^{(k)}, y^{(k)}_*} + \sum_{i \in S^{(k)}} c_i \big) \, ,
\end{align*}
where $n_t$ is the number of test samples; $y^{(k)}$ and $y^{(k)}_*$ is the $k$-th true test class and predicted test class, respectively; $S^{(k)}$ is the set of covariates that were used by the prediction model for classifying the $k$-th sample.

\subsection{Results}

For each dataset we use 5-fold cross-validation and report mean and standard deviation of the total costs. We evaluate all methods on two settings:
\begin{itemize}
	\item user-specified false positive and false negative misclassification costs.
	\item user-specified false positive misclassification cost and target recall.
\end{itemize}

If not stated otherwise, we use group lasso, as explained in Section \ref{sec:groupLasso}, for acquiring the sets $S_1 \subset S_2, \ldots S_q$.

\subsubsection{User-specified misclassification costs}

In the first setting, we assume that the user specifies the false positive cost in $\{1, 5, 10, 50, 100, 500, 1000\}$. The false negative cost is set to be 10 times the false positive cost, which reflects that it is more important to detect infected patients than avoiding wrongly classifying healthy patients.

The total cost of misclassification is shown in the top plot of
Figures \ref{fig:result_pima_asymmetricCost}, \ref{fig:result_breastcancer_asymmetricCost}, \ref{fig:result_pyhsioNetWithMissing_asymmetricCost}, and  \ref{fig:result_heartDiseaseWithMissing_asymmetricCost} for Diabetes, Breast Cancer, PhysioNet
and Heart-disease, respectively.
We observe that with respect to minimizing the total cost of classification (top plots), our proposed method AdaCOS performs better than all previously proposed methods. 

In each of those figures, in the middle and bottom plot, we also show the weighted accuracy and the number of acquired covariates each plotted against the false positive cost (which is set by the user), respectively. Since we assume that false negative classification have 10 times higher cost than false positive classification, we use the weighted accuracy defined by
\begin{align*}
\text{weighted accuracy} = \frac{\text{true positives $\cdot$ 10 + true negatives}}{\text{number of true labels $\cdot$ 10 + number of false labels}} \, .
\end{align*}

From the bottom plots, as expected, we see that all methods start acquiring more covariates as the user-specified false positive cost increases. 
At the same time, all methods, except Shim2018, show an increase in (weighted) accuracy.
In particular, Shim2018 underperforms on the smaller datasets Diabetes, Breast Cancer, and Heart-disease, which is likely due to the difficulty of adjusting the hyper-parameters of their deep neural network classifier on small hold-out validation data.

In terms of (weighted) accuracy, Full Model performs always optimal, i.e. even for small datasets we do not find any gains in predictive accuracy by using a sparser model. As a conclusion, if the covariate costs are zero or negligible, we might just opt for the full model to get the lowest total costs. On the other hand, if the ratio of false-positive cost to covariate cost is less than around 100, the full model performs considerably worse than the proposed method in terms of total cost.

The (weighted) accuracy of AdaCOS and COS are similar, while the former achieves the same accuracy with less covariates. This demonstrates the effectiveness of estimating the expected cost of misclassification depending on what we have observed so far using the conditional Bayes risk as in Equation \eqref{eq:definitionF}.

\begin{figure*}[t]
	\centering
	\caption{ Results on Pima Diabetes dataset with user-specified false positive cost in $\{1, 5, 10, 50, 100, 500, 1000\}$. The false negative cost is set to be 10 times the false positive cost. \label{fig:result_pima_asymmetricCost}}
	\includegraphics[trim=45 0 0 80,clip=true,scale=0.60,page=1]{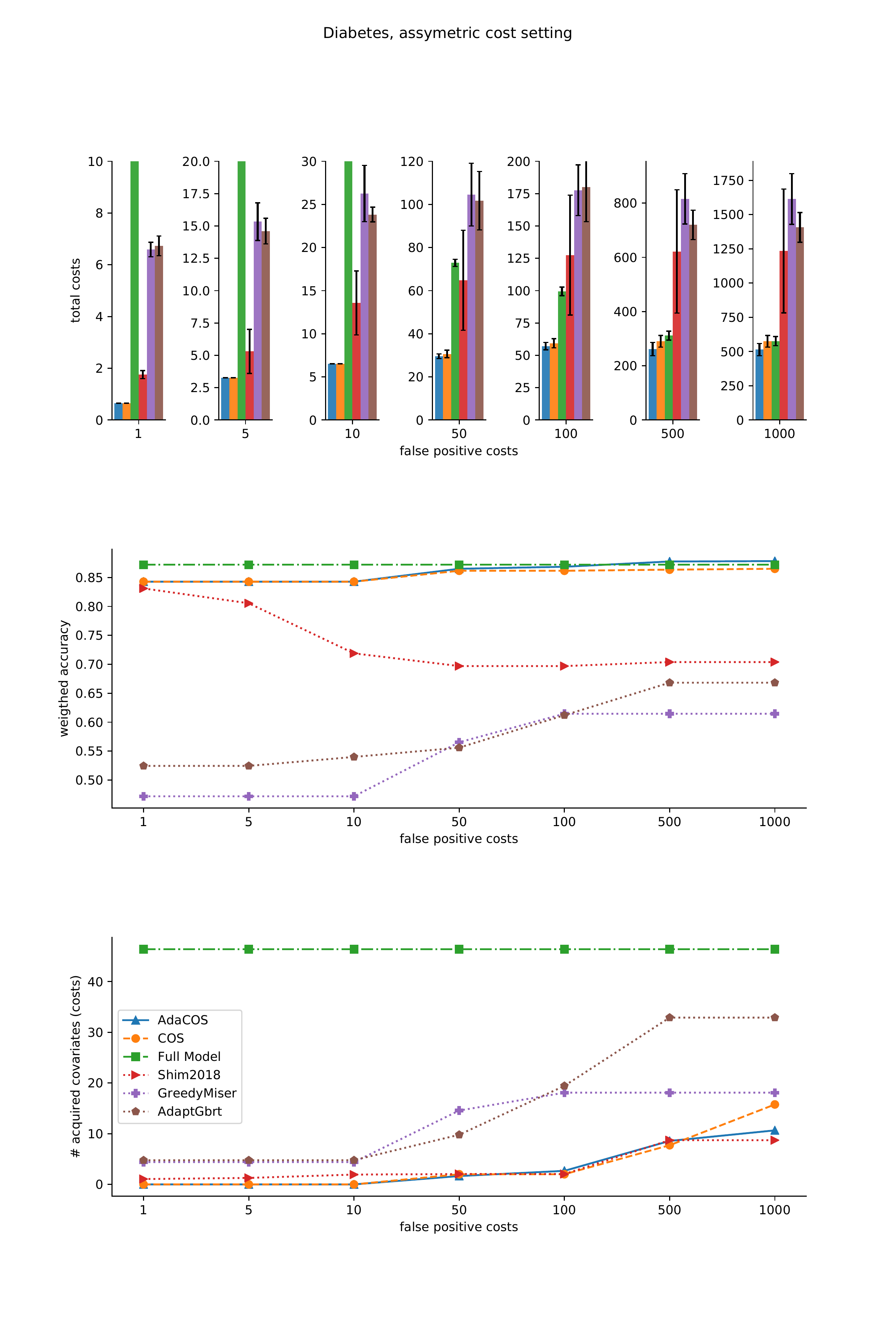}
\end{figure*}

\begin{figure*}[t]
	\centering
	\caption{ Results on Breast Cancer dataset with user-specified false positive cost in $\{1, 5, 10, 50, 100, 500, 1000\}$. The false negative cost is set to be 10 times the false positive cost. \label{fig:result_breastcancer_asymmetricCost}}
	\includegraphics[trim=45 0 0 80,clip=true,scale=0.60,page=1]{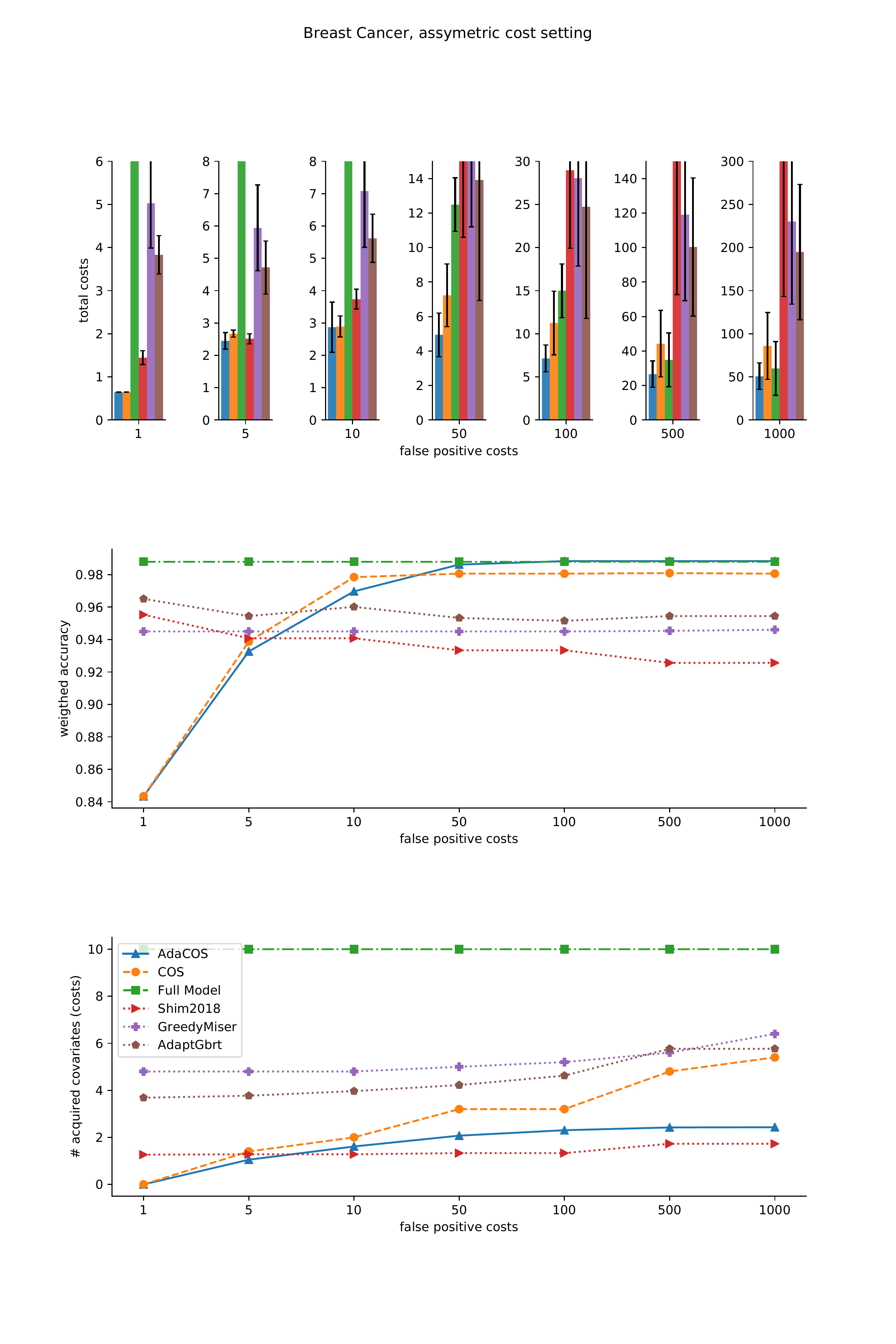}
\end{figure*}

\begin{figure*}[t]
	\centering
	\caption{ Results on PhysioNet dataset with user-specified false positive cost in $\{1, 5, 10, 50, 100, 500, 1000\}$. The false negative cost is set to be 10 times the false positive cost. \label{fig:result_pyhsioNetWithMissing_asymmetricCost}}
	\includegraphics[trim=45 0 0 80,clip=true,scale=0.60,page=1]{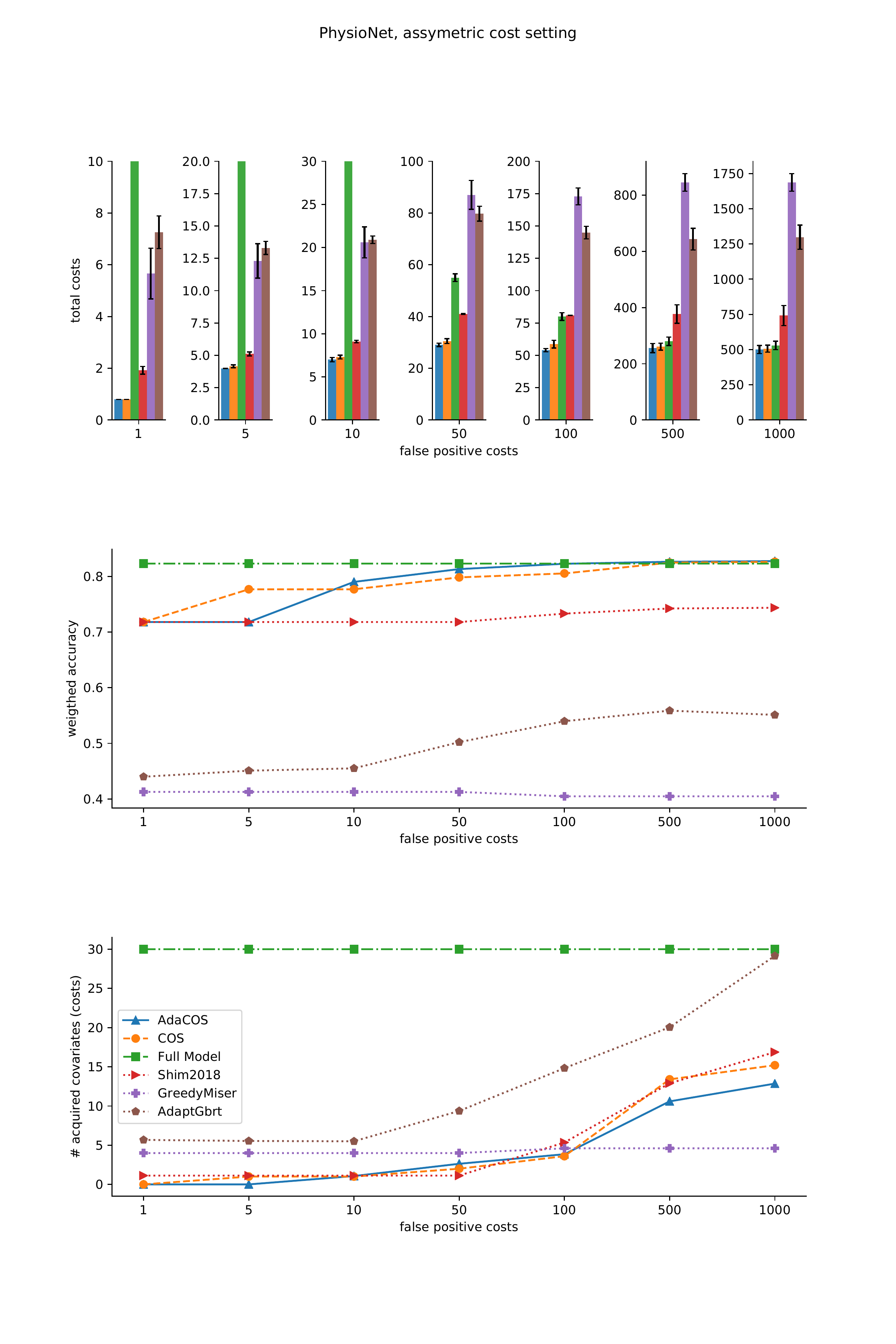}
\end{figure*}

\begin{figure*}[t]
	\centering
	\caption{ Results on Heart-disease dataset with user-specified false positive cost in $\{1, 5, 10, 50, 100, 500, 1000\}$. The false negative cost is set to be 10 times the false positive cost. \label{fig:result_heartDiseaseWithMissing_asymmetricCost}}
	\includegraphics[trim=45 0 0 80,clip=true,scale=0.60,page=1]{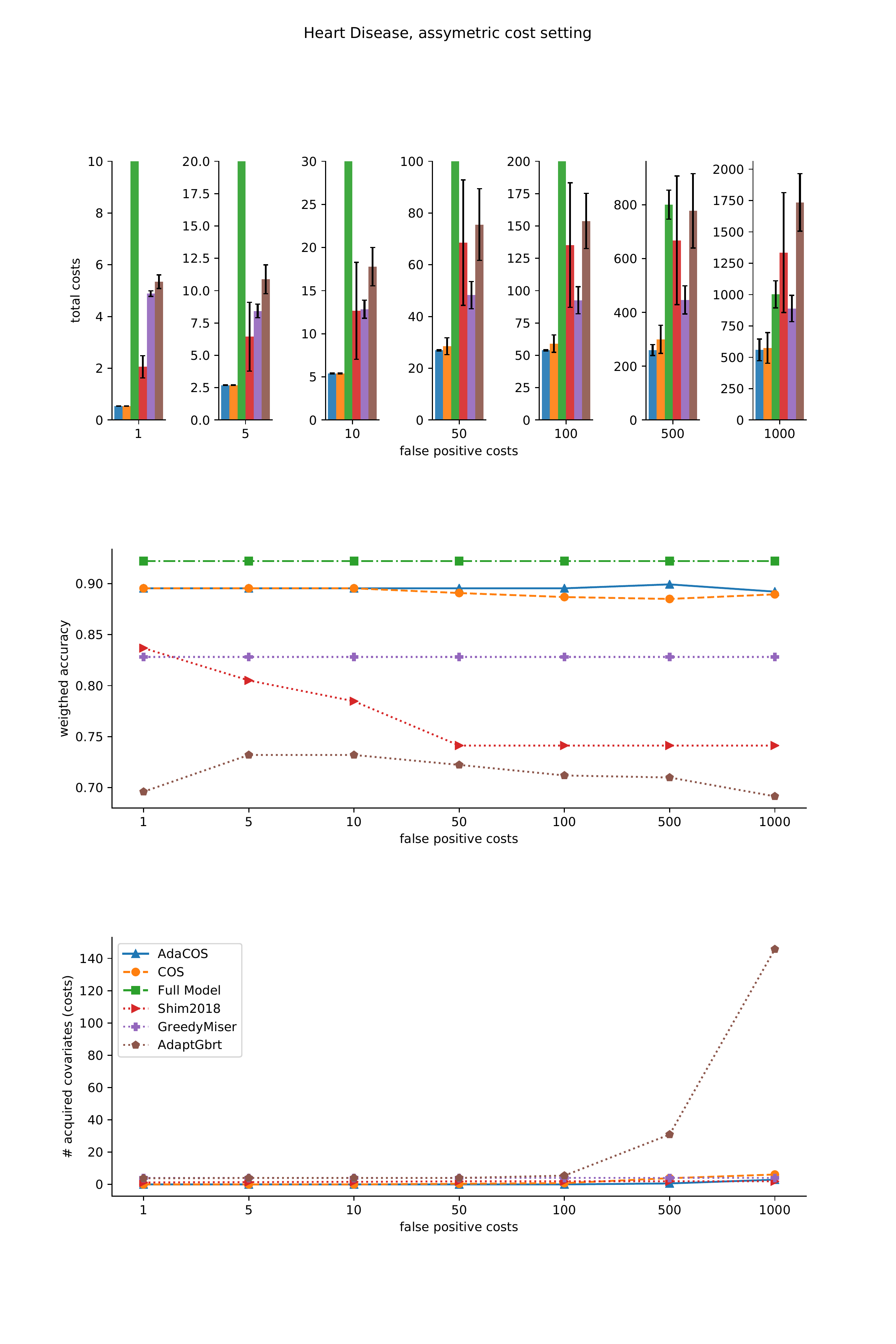}
\end{figure*}

\paragraph{Comparison of Group Lasso and Forward Selection}
Next, we compare the forward selection strategy (Section \ref{sec:forwardSelection}) and group lasso (Section \ref{sec:groupLasso}) when used for acquiring the covariate sets $S_1 \subset S_2, \ldots S_q$. 
For Diabetes, Breast Cancer, and Heart-disease, the total costs of the proposed method AdaCOS with forward selection and group lasso are shown in Figure \ref{fig:result_cmp_l1_and_greedy}. Due to the high computational costs for large $p$, it was not feasible to apply forward selection to the PhysioNet dataset.
We find that, except for the Breast Cancer dataset, both covariate acquisition strategies lead to similar results. For Breast Cancer, forward selection appears to be superior to group lasso. 

\begin{figure*}[t]
	\centering
	\caption{Comparison of the proposed method AdaCOS when defining the covariates sets $S_1, S_2, \ldots S_q$ using either the group lasso penalty or forward selection. \label{fig:result_cmp_l1_and_greedy}}
	\includegraphics[trim=50 0 0 80,clip=true,scale=0.60,page=1]{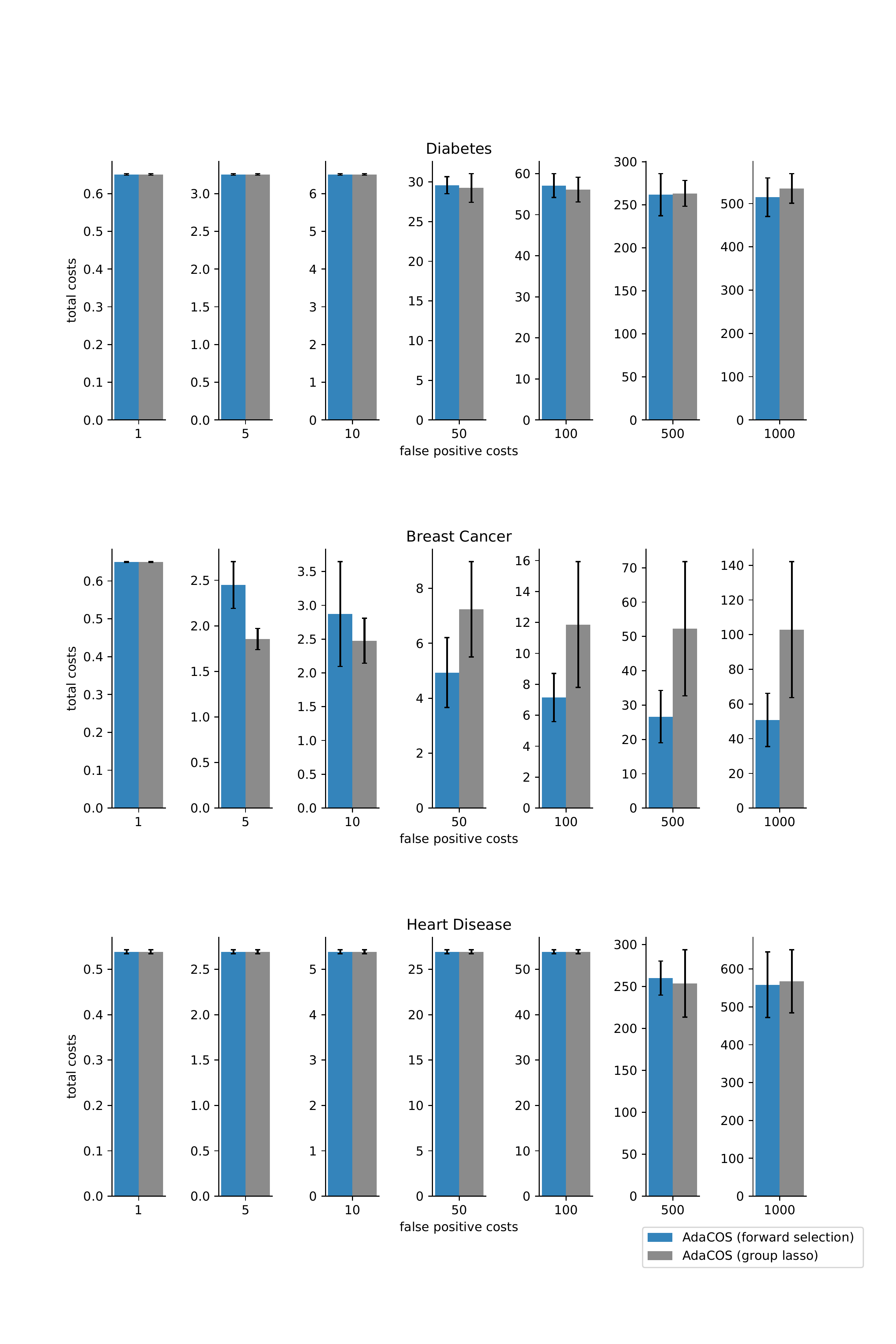}
\end{figure*}

\paragraph{Symmetric misclassification costs on Diabetes dataset}

In order to compare to the results reported in \citep{ji2007cost,dulac2012sequential},
we also evaluate on the Diabetes dataset with symmetric misclassification costs (i.e. false negative and false positive costs are the same), and the cost for correct classification set to $-50$. The results, shown in Table \ref{tab:resultsPimaData_symmetric}, suggest that also in this setting the proposed method can have an advantage over previously proposed methods. In particular, the proposed method AdaCOS with forward selection has the lowest total costs, though, when using group lasso the proposed method underperforms.

\begin{table}
	\caption{Shows the total cost of misclassification under the same cost setting as in \citep{ji2007cost,dulac2012sequential}: user-specified misclassification costs are symmetric (either 400 or 800), cost of correct classification equals $-50$. 
		The results for the methods DWSM and POMDP are taken from \citep{dulac2012sequential} and  \citep{ji2007cost}, respectively.  \label{tab:resultsPimaData_symmetric}}	
	\centering
	\begin{tabular}{lll}
		\\
		& \bf 400 & \bf 800 \\
		\midrule
		\bf AdaCOS (forward selection) & {\bf68.83} (12.43) & {\bf157.35} (24.34) \\
		\bf COS (forward selection) & 70.81 (16.75) & 161.18 (29.58) \\
		\bf AdaCOS (group lasso) & 78.1 (5.96) & 180.64 (14.12) \\
		\bf COS (group lasso) & 82.55 (12.84) & 171.71 (16.1) \\
		\bf Full Model & 99.52 (8.67) & 191.2 (16.38) \\
		\bf Shim2018 & 113.82 (17.87) & 246.71 (26.71) \\
		\bf AdaptGBRT & 87.41 (12.1) & 175.8 (17.76) \\
		\bf GreedyMiser & 91.36 (14.43)  & 200.96 (31.93)  \\
		\bf DWSM & 74.0 (-) & 181.0 (-) \\
		\bf POMDP & 75.0 (-) & 180.0 (-) \\
		\bottomrule
	\end{tabular}
\end{table}

\subsubsection{User-specified target recall}

Finally, we investigate the setting where the user specifies target recall $r$ instead of false negative costs. Here, we show the results for $r = 0.95$, 
the results for target recall $r = 0.99$ are similar and given in the supplement material.

For this setting, we do not consider the method AdaptGbrt, since it does not allow the output of class probabilities or scores.
For Shim2018 and GreedyMiser, we found that simply using the class probabilities/scores from the validation data to learn thresholds with recall $\geq r$, tended to lead to recall less that $r$ on the test data, as shown in Figure \ref{fig:result_onlyRecall95}. 
Therefore, in order to make all results comparable, we show the results for Shim2018 and GreedyMiser at the same recall level as the proposed method AdaCOS.

The recall of the proposed method AdaCOS on the test data, as shown in Figure \ref{fig:result_onlyRecall95}, never violates the target recall of $0.95$.

Since no false negative costs are provided, we cannot evaluate in terms of total costs anymore. Instead, we evaluate in terms of average operation costs, defined as 
the average cost of false positives plus the costs for covariate acquisition:
\begin{align*} 
\text{avg operation costs} := \frac{1}{n_t} \sum_{k = 1}^{n_t} \big( (1 - y^{(k)}) \cdot y^{(k)}_*  \cdot c_{0,1} + \sum_{i \in S^{(k)}} c_i \big) \, .
\end{align*}

The results for all datasets are shown in Figures \ref{fig:result_pima_recall95}, \ref{fig:result_breastcancer_recall95}, \ref{fig:result_pyhsioNetWithMissing_recall95}, and \ref{fig:result_heartDiseaseWithMissing_recall95}.
For completeness, on each of the those figures, we also plot the false discovery rate (FDR) against the covariate costs (bottom plots), where the crosses of different sizes mark the standard deviation of FDR and covariates costs.

For all datasets, except Heart-Disease, the proposed method AdaCOS has the smallest operation costs. Furthermore, AdaCOS tends to achieve a lower false discovery rate with less covariates used. 

\begin{figure*}[t]
	\centering
	\caption{Results on Pima Diabetes dataset with user-specified false positive cost in $\{1, 5, 10, 50, 100, 500, 1000\}$, and target recall $\geq 0.95$. \label{fig:result_pima_recall95}}
	\includegraphics[trim=45 0 0 80,clip=true,scale=0.60,page=1]{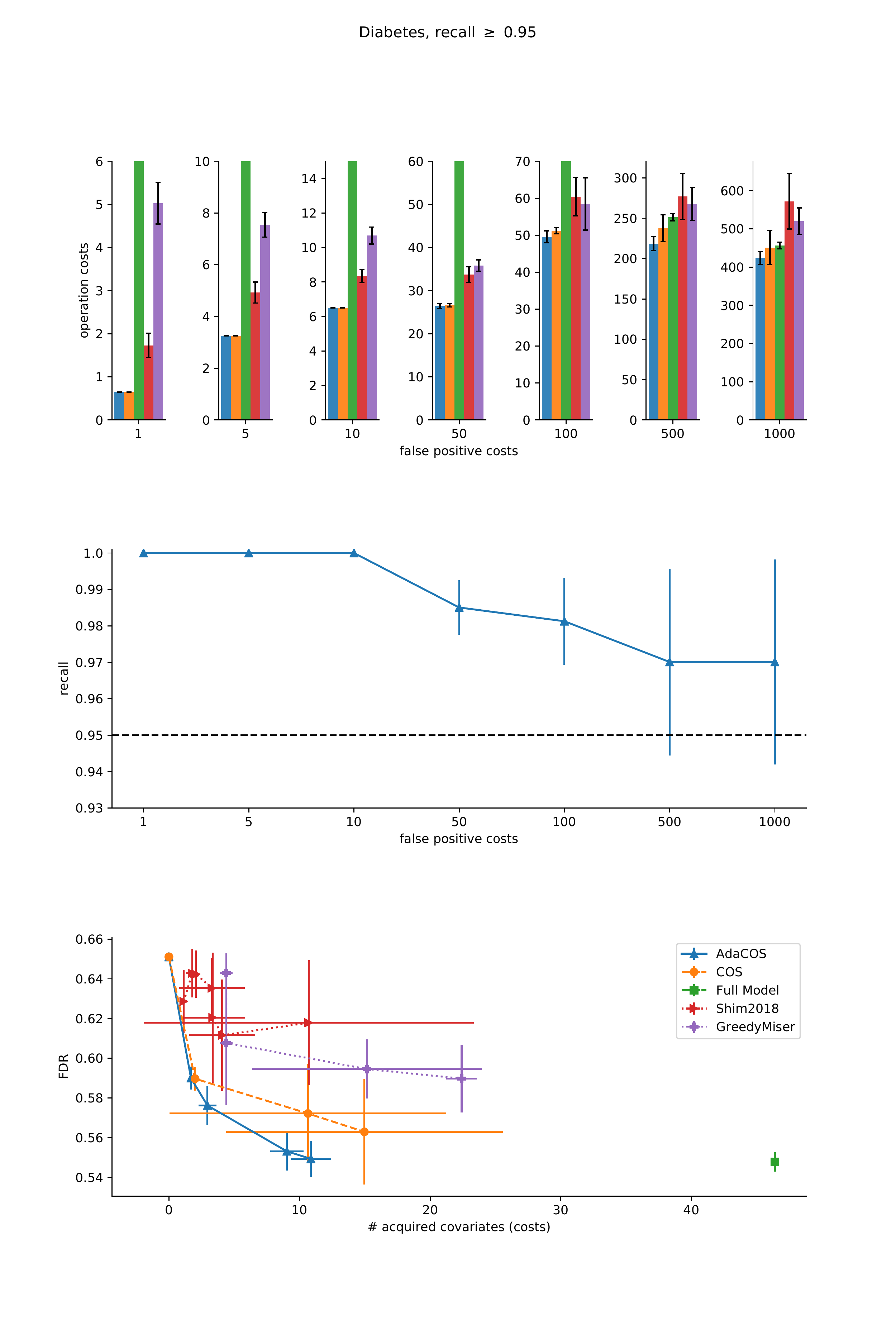}
\end{figure*}

\begin{figure*}[t]
	\centering
	\caption{Results on Breast Cancer dataset with user-specified false positive cost in $\{1, 5, 10, 50, 100, 500, 1000\}$, and target recall $\geq 0.95$. \label{fig:result_breastcancer_recall95}}
	\includegraphics[trim=45 0 0 80,clip=true,scale=0.60,page=1]{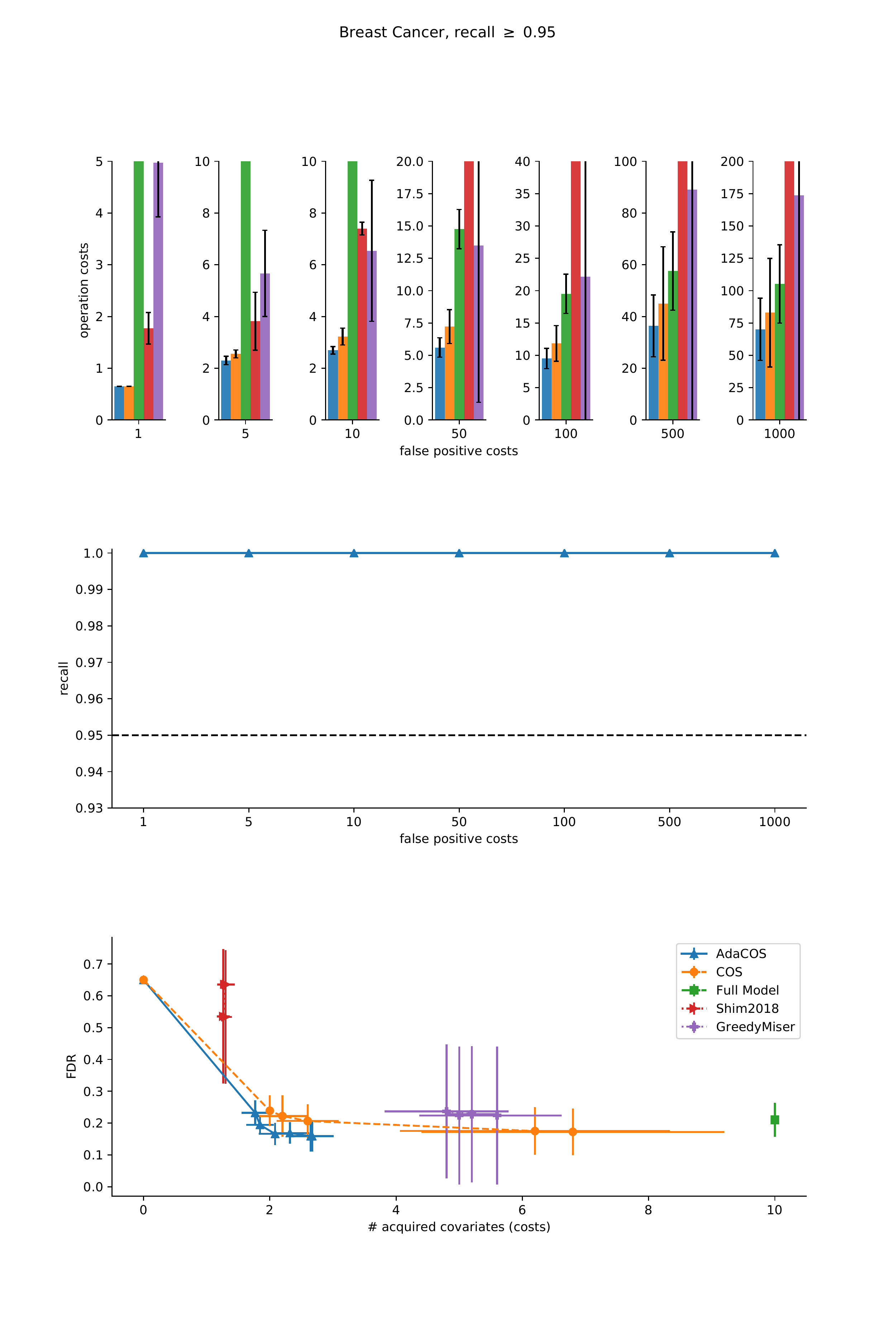}
\end{figure*}

\begin{figure*}[t]
	\centering
	\caption{Results on PhysioNet dataset with user-specified false positive cost in $\{1, 5, 10, 50, 100, 500, 1000\}$, and target recall $\geq 0.95$. \label{fig:result_pyhsioNetWithMissing_recall95}}
	\includegraphics[trim=45 0 0 80,clip=true,scale=0.60,page=1]{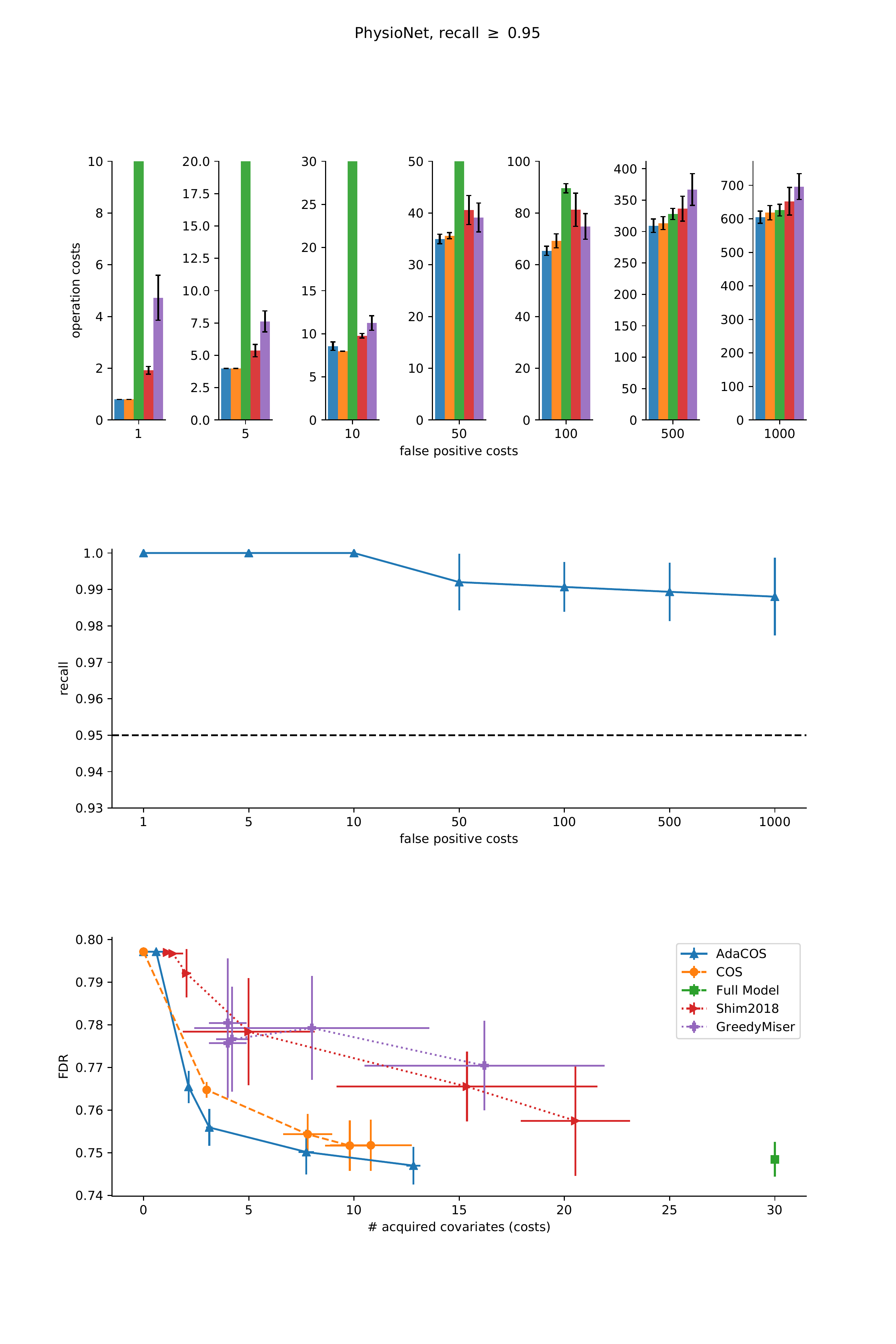}
\end{figure*}

\begin{figure*}[t]
	\centering
	\caption{Results on Heart-disease dataset with user-specified false positive cost in $\{1, 5, 10, 50, 100, 500, 1000\}$, and target recall $\geq 0.95$. \label{fig:result_heartDiseaseWithMissing_recall95}}
	\includegraphics[trim=45 0 0 80,clip=true,scale=0.60,page=1]{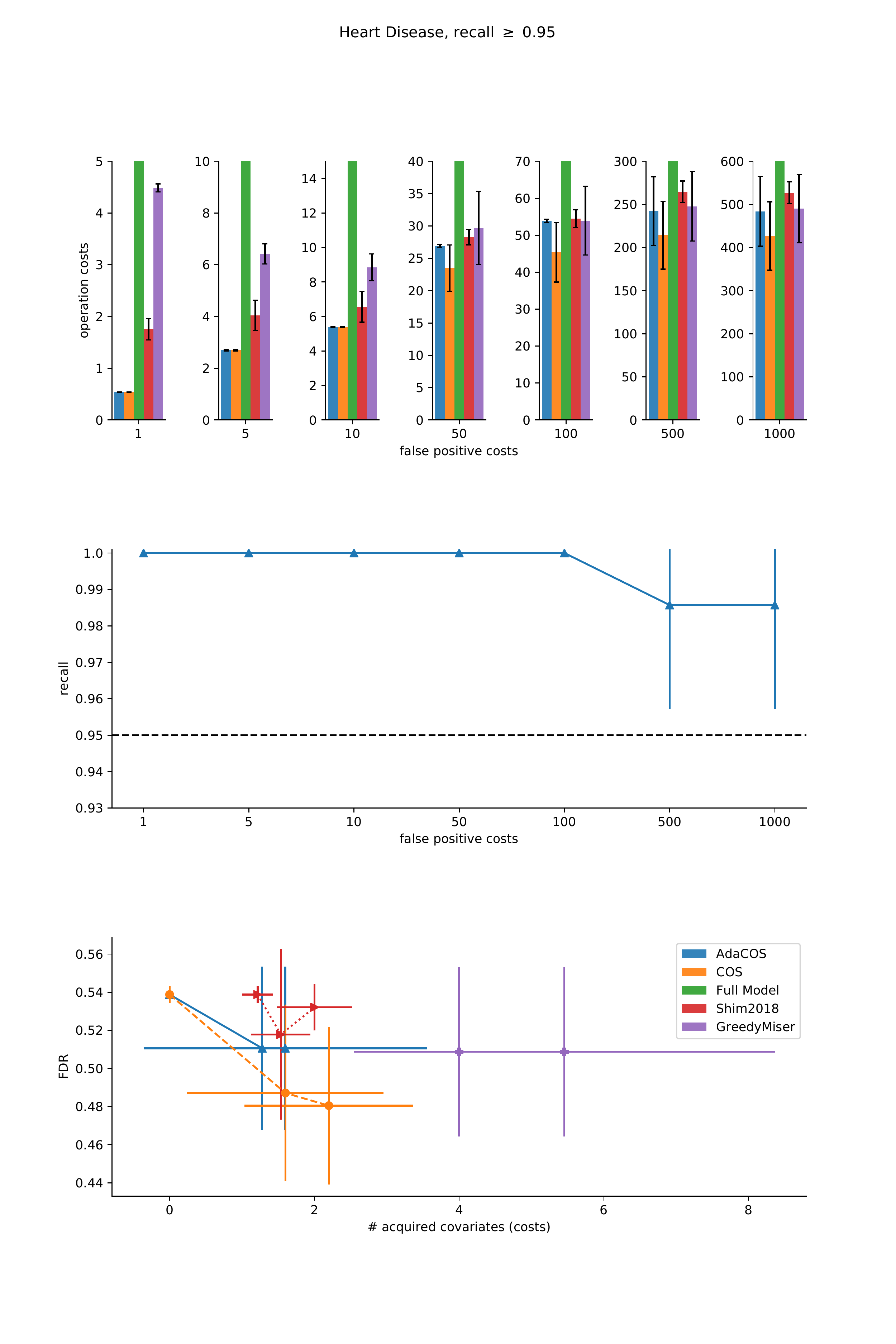}
\end{figure*}

\begin{figure*}[t]
	\centering
	\caption{Actually observed recall on test data when threshold on class probabilities was adjusted on validation data to have recall $\geq 0.95$. \label{fig:result_onlyRecall95}}
	\includegraphics[trim=45 0 0 80,clip=true,scale=0.60,page=1]{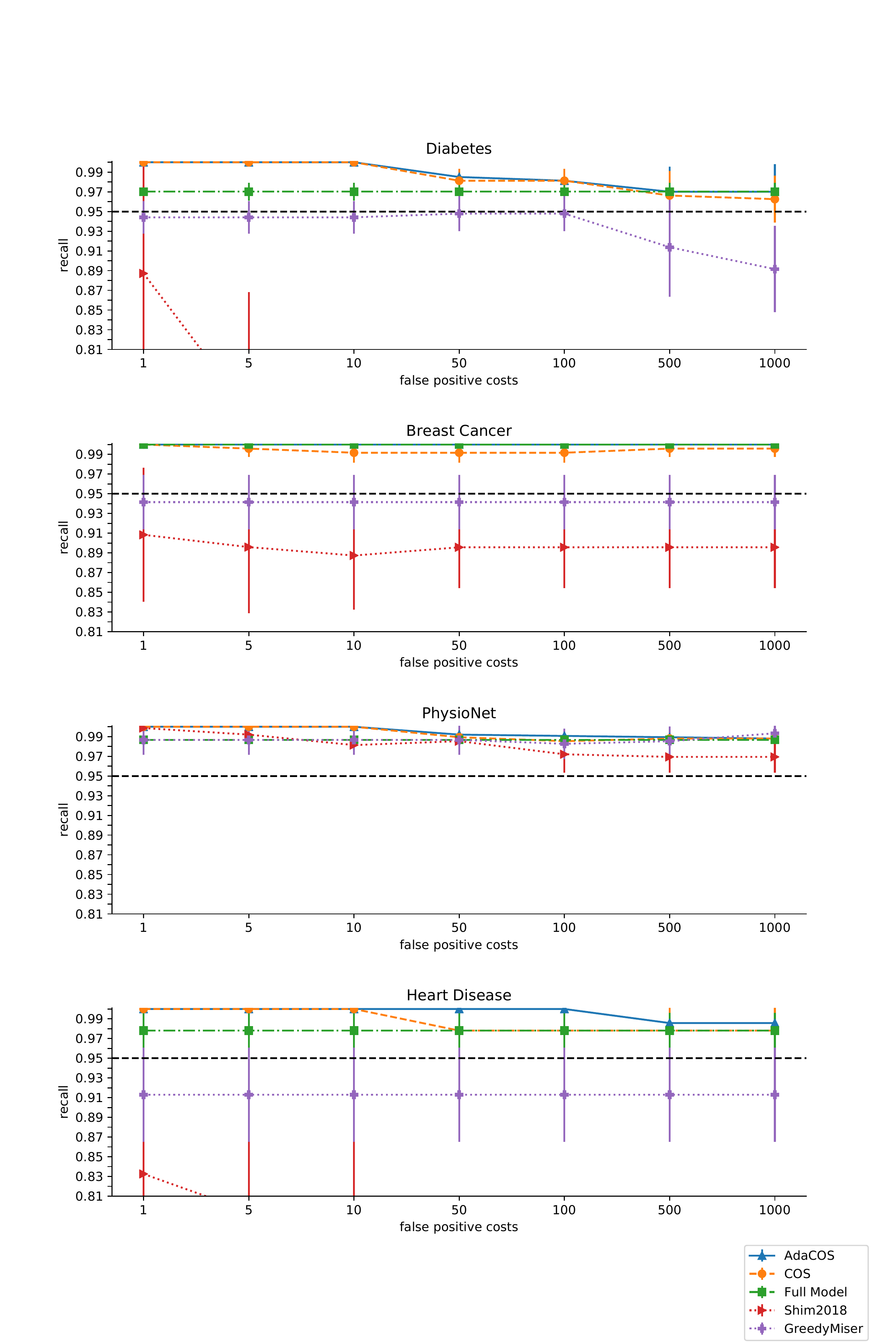}
\end{figure*}

\section{Related Work} \label{sec:relatedWork}
Here, we briefly summarize various previous works for cost-sensitive classification.

\paragraph{Markov Decision Process (MDP) Framework}
The MDP formulation and solution using an action-utility representation (Q-learning) in \citep{zubek2004pruning,bayer2004learning} is closest to our approach. 
Their method also leads to a Bayes procedure. However, they do not provide a formal proof and consider only discrete covariates. 
The work in \citep{dulac2011datum,dulac2012sequential,karayev2013dynamic} also uses the MDP framework.
However, their proposed method cannot incorporate the uncertainty about the covariate distributions.
%
The work in \citep{ji2007cost} tries to model such uncertainties by modeling the cost-sensitive classification problem as a partial observable Markov decision process (POMDP). However, their POMDP formulation can lead to repeatedly selecting the same covariates, and as a consequence they need to adapt the stopping criteria. 

\paragraph{Reinforcement Learning Approaches}
\cite{janisch2017classification,shim2018joint} suggest to use deep reinforcement learning with Q-learning.
In contrast to MDP, a discriminative decision maker is learned which does not require an environmental model. Their method performs promising in the domain where huge amounts of labeled training data is available.
Alternatively, the work in \citep{benbouzid2012fast} suggests the use of SARSA.
The method in \citep{contardo2016sequential} also addresses this problem with reinforcement learning. 


\paragraph{Discriminative Decision Approach}
The work in \citep{wang2015efficient} proposes an intriguing method for finding a decision procedure that is guaranteed to converge to the Bayes risk given sufficient enough training data. 
Their idea is to create a Bayes optimal classifier for all possible subsets of covariates, and a directed a-cyclic graph that connects them. 
They formulate the problem as an empirical risk minimization (ERM) problem, and show that with infinitely many training samples the loss at each node converges to the Bayes risks.  However, in order to allow for scalability their method requires to acquire covariates in batches.
The work in \citep{trapeznikov2013supervised,wang2014lp} uses a similar framework but is restricted to a fixed sequential order.

\paragraph{Cost-sensitive Tree Construction}
The work in \citep{xu2012greedy,nan2015feature,nan2016pruning,nan2017adaptive,peter2017cost} learns a random forest subject to budget constraints on the features. 
In particular, the methods in \citep{nan2017adaptive,peter2017cost} are considered state of the art for this task. Their usage of gradient boosted decision trees \citep{friedman2001greedy} makes them in particular effective for very large training data.
Cost-sensitive decision trees for discrete covariates are also considered in \citep{sheng2006feature}, and extended to Bayesian Networks in \citep{bilgic2007voila}.


\paragraph{Tree of Classifiers}
The work in \citep{kusner2014feature,xu2013cost} proposes to learn a tree of classifiers that minimizes a convex surrogate loss subject to budget constraints. 
\cite{wang2014model} assumes a fixed number of pre-trained classifiers and the goal is to learn a policy that selects one of those classifiers.

\paragraph{Entropy-Based Approaches}
The work in \citep{kanani2008prediction,gao2011active,kapoor2009breaking,gong2019icebreaker} optimizes a criteria that combines the costs of features with an estimate of the class entropy of the resulting classifier. As such their objective function is different from ours. 

\paragraph{Density Estimation via Autoencoders}

The work in \citep{kachuee2018dynamic} suggests to acquire the covariate which has the highest sensitivity to the output prediction $y$. In order to account for different covariate acquisition costs the sensitivity scores are re-scaled appropriately. The sensitivity scores are estimated using a denoising autoencoder. 
Similarly, the work in \citep{ma2019eddi} uses as objective function the expected Shannon information, which is estimated via a variational autoencoder.
Both objective functions are not related to the minimization of the expected total cost.


\paragraph{Others}
The work in \citep{greiner2002learning} extends the Probably Approximately Correct (PAC) framework to prove the existence of a cost-sensitive classifier that is 
with high probability optimal in the sense of providing minimal average total costs. However, they assume a probability distribution over only discrete covariates. 
The method in \citep{lakkaraju2017learning} is additionally focused on interpretability, and, as a consequence, optimizes an objective function that is different from ours.
Imitation learning is also applied to this task by \cite{he2012imitation}, but their definition of loss is different from minimizing the total classification costs that we consider here.
The work in \citep{nan2014fast} assumes a margin-based classifier and uses a k-nearest neighbor approach to estimate the accuracy of the classifier.

\section{Conclusions} \label{sec:conclusions}

In this article, we addressed the problem of cost-sensitive classification where the goal is to minimize the total costs, defined as the expected cost of misclassification plus the cost for covariate acquisition. 

In Section \ref{sec:costRationalSelectionCriteria}, we rigorously formalized this goal as the minimization of the (conditional) Bayes risk which can change after the acquisition of a new covariate. 
However, solving this minimization problem is hard.
First, the evaluation of the conditional Bayes risk requires to estimate and integrate over a high dimensional density. Second, the Bayes risk must be evaluated for all combinations of covariate sets which is exponential in $p$ the number of covariates.

In order to overcome the computational difficulties, we introduced two working assumptions:
\begin{enumerate}
	\item The optimal classifier can be expressed as a generalized additive model (GAM).
	\item The optimal sets of covariates can expressed as a sequence of sets that are monotone increasing, namely $S_1 \subset S_2 \ldots \subset S_q$.
\end{enumerate}

Using the first assumption, we showed, in Section \ref{sec:AdaCOS}, that the evaluation of the conditional Bayes risk reduces to a one dimensional density estimation and integration problem which can be efficiently estimated.

Furthermore, we showed that the sequence $S_1 \subset S_2 \ldots \subset S_q$ can be computationally efficiently acquired by inspecting the regression coefficient path when penalizing GAM with group lasso.

Our experiments suggest that our proposed method AdaCOS achieves in most situations the lowest total costs of classification, when compared to the previous methods POMDP, DWSM, GreedyMiser, AdaptGbrt, and Shim2018 \citep{ji2007cost,dulac2012sequential,xu2012greedy,nan2017adaptive,shim2018joint}.

We note that some previous methods like Shim2018 \citep{shim2018joint} do not share our working assumptions, and instead use a very flexible classifier (deep neural network) and covariate acquisition strategy based on reinforcement learning.
However, for small datasets, and even for medium large datasets like PhysioNet, we found that a generalized additive model is competitive or even better than a neural network classifier,
and the flexibility of the reinforcement learning seems to suffer from high variance. 

Finally, we considered the situation where not all misclassification costs are specified by the user. In particular, we considered the situation where the user specifies a target recall instead of the cost of false negative classification. We showed that it is possible to apply the proposed method by estimating the implicitly defined false negative cost. Our experiments showed that the resulting method indeed achieves the desired minimum recall, while minimizing the false discovery rate and covariate acquisition cost.

The source code of the proposed method and for reproducing all results is available at \url{https://github.com/andrade-stats/AdaCOS_public}.

\section*{Appendix}


Here, we prove Theorem \ref{thm:optDecisionProcedure}, which states that 
the procedure $\pi^*$, as defined in Equation \eqref{eq:optDecisionProcedure}, is a Bayes procedure.
That means we need to show that for any other decision procedure $\pi$ we have
\begin{align*} 
\E_{\mathbf{x}, y} [l( (\mathbf{x}, y), \pi^*) ] \leq \E_{\mathbf{x}, y} [l( (\mathbf{x}, y), \pi) ] \, .
\end{align*}

\begin{proof}
	Let $S \subseteq V$ be the set of already observed covariates, then the expected remaining costs for a decision procedure $\pi$ is given by
	\begin{align*}
	\E_{\mathbf{x}_{V \setminus S}, y} [l( (\mathbf{x}, y), \pi, S) | \mathbf{x}_S ] \, .
	\end{align*}
	We will prove by induction that for any $S \subseteq V$, and any decision procedure $\pi$, we have
	\begin{align*}
	\E_{\mathbf{x}_{V \setminus S}, y} [l( (\mathbf{x}, y), \pi^*, S) | \mathbf{x}_S ] \leq \E_{\mathbf{x}_{V \setminus S}, y} [l( (\mathbf{x}, y), \pi, S) | \mathbf{x}_S ] \, .
	\end{align*}
	The claim then follows by setting $S := \emptyset$.
	\item
	\paragraph{Base case: $S = V$. }
	We have 
	\begin{align*} 
	\pi^*(\mathbf{x}_S) = \argmin_{i \in L} \E_y [c_{y,i} |  \mathbf{x}_S] \, .
	\end{align*}
	Therefore $\pi^*(\mathbf{x}_S)$ is a Bayes procedure, and as a consequence 
	\begin{align*} 
	\E_{y} [c_{y,\pi^*(\mathbf{x}_S)} | \mathbf{x}_S ] \leq \E_y [c_{y,\pi(\mathbf{x}_S)} |  \mathbf{x}_S]  \, .
	\end{align*}
	And therefore 
	\begin{align*}
	\E_{\mathbf{x}_{V \setminus S}, y} [l( (\mathbf{x}, y), \pi^*, S) | \mathbf{x}_S ] \leq \E_{\mathbf{x}_{V \setminus S}, y} [l( (\mathbf{x}, y), \pi, S) | \mathbf{x}_S ] \, .
	\end{align*}
	(Since $S = V$, and we have
	\begin{align*} 
	\E_{\mathbf{x}_{V \setminus S}, y} [l( (\mathbf{x}, y), \pi^*, S) | \mathbf{x}_S ] = \E_{y} [l( (\mathbf{x}, y), \pi^*, S) | \mathbf{x}_S ] = \E_{y} [c_{y,\pi^*(S)} | \mathbf{x}_S ] \, ,
	\end{align*}
	and the same analogously for $\pi$.)
	\item
	\paragraph{Induction step: $S \subset V$.}
	Assume that for all $S \cup \{i\}$, where $i \in V \setminus S$, the induction assumptions holds, that is 
	\begin{align*}
	\E_{\mathbf{x}_{V \setminus (S \cup \{i\})}, y} [l( (\mathbf{x}, y), \pi^*, S \cup \{i\}) | \mathbf{x}_{S \cup \{i\}}] \leq \E_{\mathbf{x}_{V \setminus (S \cup \{i\})}, y} [l( (\mathbf{x}, y), \pi, S \cup \{i\}) | \mathbf{x}_{S \cup \{i\}}]  \, .
	\end{align*}
	
	Let $\hat{\pi}$ denote a Bayes procedure. Using the structure of the loss function as defined in Equation \eqref{eq:LossDefinition},
	we have
	\begin{align*}
	&\E_{\mathbf{x}_{V \setminus S}, y} [l( (\mathbf{x}, y), \hat{\pi}, S) | \mathbf{x}_S ] \\
	&= \E_{\mathbf{x}_{V \setminus S}, y} [
	\begin{cases}
	c_{y,\hat{\pi}(\mathbf{x})} & \text{if } \hat{\pi}(\mathbf{x}) \in L \, ,  \\
	c_{\hat{\pi}(\mathbf{x})} + l((\mathbf{x}, y), \hat{\pi}, S \cup \{ \hat{\pi}(\mathbf{x}) \}) & \text{else.}
	\end{cases}
	| \mathbf{x}_S ] \\
	&\geq \min_{i \in L \cup (V \setminus S)} \E_{\mathbf{x}_{V \setminus S}, y} [
	\begin{cases}
	c_{y,i} & \text{if } i \in L \, ,  \\
	c_i + l((\mathbf{x}, y), \hat{\pi}, S \cup \{i\}) & \text{else.}
	\end{cases}
	| \mathbf{x}_S ] \\
	&= \min_{i \in L \cup (V \setminus S)} 
	\begin{cases}
	\E_{\mathbf{x}_{V \setminus S}, y} [c_{y,i} | \mathbf{x}_S ] & \text{if } i \in L \, ,  \\
	c_i + \E_{\mathbf{x}_{V \setminus S}, y} [l((\mathbf{x}, y), \hat{\pi}, S \cup \{i\}) | \mathbf{x}_S ] & \text{else.}
	\end{cases} \\
	&= \min_{i \in L \cup (V \setminus S)} 
	\begin{cases}
	\E_{\mathbf{x}_{V \setminus S}, y} [c_{y,i} | \mathbf{x}_S ] & \text{if } i \in L \, ,  \\
	c_i + \E_{x_i} \Big[ \E_{\mathbf{x}_{V \setminus (S \cup \{i\})}, y} \big[l((\mathbf{x}, y), \hat{\pi}, S \cup \{i\}) | \mathbf{x}_{S \cup \{i\}} \big] | \mathbf{x}_{S} \Big] & \text{else.}
	\end{cases} \\
	&\stackrel{(1)}{\geq} \min_{i \in L \cup (V \setminus S)} 
	\begin{cases}
	\E_{\mathbf{x}_{V \setminus S}, y} [c_{y,i} | \mathbf{x}_S ] & \text{if } i \in L \, ,  \\
	c_i + \E_{x_i} \Big[ \E_{\mathbf{x}_{V \setminus (S \cup \{i\})}, y} \big[l((\mathbf{x}, y), \pi^*, S \cup \{i\}) | \mathbf{x}_{S \cup \{i\}} \big] | \mathbf{x}_{S} \Big] & \text{else.}
	\end{cases} \\
	&= \min_{i \in L \cup (V \setminus S)} 
	\begin{cases}
	\E_{\mathbf{x}_{V \setminus S}, y} [c_{y,i} | \mathbf{x}_S ] & \text{if } i \in L \, ,  \\
	c_i + \E_{\mathbf{x}_{V \setminus S}, y} \big[l((\mathbf{x}, y), \pi^*, S \cup \{i\}) | \mathbf{x}_{S} \big] & \text{else.}
	\end{cases} \\
	&=  \E_{\mathbf{x}_{V \setminus S}, y} [l( (\mathbf{x}, y), \pi^*, S) | \mathbf{x}_S ] \, ,
	\end{align*}
	where in the line marked by (1) we used the induction assumption. The last line follows from Lemma \ref{lemma:equalityHelper}.
	Since $\hat{\pi}$ is a Bayes procedure, we must have equality in the second and fifth line. Therefore $\pi^*$ is also a Bayes procedure.
\end{proof}

\begin{lemma} \label{lemma:equalityHelper}
	\begin{align*}
	&\E_{\mathbf{x}_{V \setminus S}, y} [l( (\mathbf{x}, y), \pi^*, S) | \mathbf{x}_S ] \\
	&= \min_{i \in L \cup (V \setminus S)} 
	\begin{cases}
	\E_{\mathbf{x}_{V \setminus S}, y} [c_{y,i} | \mathbf{x}_S ] & \text{if } i \in L \, ,  \\
	c_i + \E_{\mathbf{x}_{V \setminus S}, y} \big[l((\mathbf{x}, y), \pi^*, S \cup \{i\}) | \mathbf{x}_{S} \big] & \text{else.}
	\end{cases} 
	\end{align*}
\end{lemma}

\begin{proof}
	\begin{align*}
	&\E_{\mathbf{x}_{V \setminus S}, y} [l( (\mathbf{x}, y), \pi^*, S) | \mathbf{x}_S ] \\
	&=  \E_{\mathbf{x}_{V \setminus S}, y} [
	\begin{cases}
	c_{y,\pi^*(\mathbf{x}_S)} & \text{if } \pi^*(\mathbf{x}_S) \in L \, ,  \\
	c_{\pi^*(\mathbf{x}_S)} + l((\mathbf{x}, y), \pi^*, S \cup \{ \pi^*(\mathbf{x}_S) \})  & \text{else.}
	\end{cases} 
	| \mathbf{x}_S ] \\
	&=  
	\begin{cases}
	\E_{\mathbf{x}_{V \setminus S}, y} [ c_{y,\pi^*(\mathbf{x}_S)} | \mathbf{x}_S ]  & \text{if } \pi^*(\mathbf{x}_S) \in L \, ,  \\
	c_{\pi^*(\mathbf{x}_S)} + \E_{\mathbf{x}_{V \setminus S}, y} [ l((\mathbf{x}, y), \pi^*, S \cup \{ \pi^*(\mathbf{x}_S) \})  | \mathbf{x}_S ]  & \text{else.}
	\end{cases} \\
	&=  
	\begin{cases}
	\E_{y} [ c_{y,\pi^*(\mathbf{x}_S)} | \mathbf{x}_S ]  & \text{if } \pi^*(\mathbf{x}_S) \in L \, ,  \\
	c_{\pi^*(\mathbf{x}_S)} + \E_{\mathbf{x}_{V \setminus S}, y} [ l((\mathbf{x}, y), \pi^*, S \cup \{\pi^*(\mathbf{x}_S)\})  | \mathbf{x}_S ]  & \text{else.}
	\end{cases} 
	\end{align*}
	\item
	\paragraph{1. Case: $\pi^*(\mathbf{x}_S) \in L$.}
	Then because of the definition of $\pi^*$, we have 
	\begin{align*}
	\E_{y} [ c_{y,\pi^*(\mathbf{x}_S)} | \mathbf{x}_S ]
	&= 
	\min_{i \in L \cup (V \setminus S)} 
	\begin{cases}
	\E_{y} [ c_{y,i} | \mathbf{x}_S ]  & \text{if } i \in L \, ,  \\
	c_{i} + \E_{\mathbf{x}_{V \setminus S}, y} [ l((\mathbf{x}, y), \pi^*, S \cup \{i\})  | \mathbf{x}_S ]  & \text{else.}
	\end{cases} 
	\end{align*}
	
	\item
	\paragraph{2. Case: $\pi^*(\mathbf{x}_S) \notin L$.}
	Then because of the definition of $\pi^*$, we have 
	\begin{align*}
	& c_{\pi^*(\mathbf{x}_S)} + \E_{\mathbf{x}_{V \setminus S}, y} [ l((\mathbf{x}, y), \pi^*, S \cup \{\pi^*(\mathbf{x}_S)\})  | \mathbf{x}_S ] 
	\\
	&= \min_{i \in L \cup (V \setminus S)} 
	\begin{cases}
	\E_{y} [ c_{y,i} | \mathbf{x}_S ]  & \text{if } i \in L \, ,  \\
	c_{i} + \E_{\mathbf{x}_{V \setminus S}, y} [ l((\mathbf{x}, y), \pi^*, S \cup \{i\})  | \mathbf{x}_S ]  & \text{else.}
	\end{cases} 
	\end{align*}
\end{proof}

\bibliographystyle{plainnat}
\bibliography{all_papers_bibliography_extended}

\begin{thebibliography}{47}
\providecommand{\natexlab}[1]{#1}
\providecommand{\url}[1]{\texttt{#1}}
\expandafter\ifx\csname urlstyle\endcsname\relax
  \providecommand{\doi}[1]{doi: #1}\else
  \providecommand{\doi}{doi: \begingroup \urlstyle{rm}\Url}\fi

\bibitem[Anderson(2003)]{anderson2003introduction}
Theodore~Wilbur Anderson.
\newblock \emph{An introduction to multivariate statistical analysis},
  volume~2.
\newblock Wiley New York, 2003.

\bibitem[Andrade and Okajima(2019)]{andrade2019efficient}
Daniel Andrade and Yuzuru Okajima.
\newblock Efficient bayes risk estimation for cost-sensitive classification.
\newblock In \emph{The 22nd International Conference on Artificial Intelligence
  and Statistics}, pages 3372--3381, 2019.

\bibitem[Bayer-Zubek(2004)]{bayer2004learning}
Valentina Bayer-Zubek.
\newblock Learning diagnostic policies from examples by systematic search.
\newblock In \emph{Proceedings of the 20th conference on Uncertainty in
  artificial intelligence}, pages 27--34. AUAI Press, 2004.

\bibitem[Benbouzid et~al.(2012)Benbouzid, Busa-Fekete, and
  K{\'e}gl]{benbouzid2012fast}
Djalel Benbouzid, R{\"o}bert Busa-Fekete, and Bal{\'a}zs K{\'e}gl.
\newblock Fast classification using sparse decision dags.
\newblock In \emph{Proceedings of the 29th International Coference on
  International Conference on Machine Learning}, pages 747--754, 2012.

\bibitem[Bilgic and Getoor(2007)]{bilgic2007voila}
Mustafa Bilgic and Lise Getoor.
\newblock Voila: Efficient feature-value acquisition for classification.
\newblock In \emph{Proceedings of the National Conference on Artificial
  Intelligence}, volume~22, page 1225. Menlo Park, CA; Cambridge, MA; London;
  AAAI Press; MIT Press; 1999, 2007.

\bibitem[Contardo et~al.(2016)Contardo, Denoyer, and
  Arti{\`e}res]{contardo2016sequential}
Gabriella Contardo, Ludovic Denoyer, and Thierry Arti{\`e}res.
\newblock Sequential cost-sensitive feature acquisition.
\newblock In \emph{International Symposium on Intelligent Data Analysis}, pages
  284--294. Springer, 2016.

\bibitem[Dulac-Arnold et~al.(2011)Dulac-Arnold, Denoyer, Preux, and
  Gallinari]{dulac2011datum}
Gabriel Dulac-Arnold, Ludovic Denoyer, Philippe Preux, and Patrick Gallinari.
\newblock Datum-wise classification: a sequential approach to sparsity.
\newblock In \emph{Joint European Conference on Machine Learning and Knowledge
  Discovery in Databases}, pages 375--390. Springer, 2011.

\bibitem[Dulac-Arnold et~al.(2012)Dulac-Arnold, Denoyer, Preux, and
  Gallinari]{dulac2012sequential}
Gabriel Dulac-Arnold, Ludovic Denoyer, Philippe Preux, and Patrick Gallinari.
\newblock Sequential approaches for learning datum-wise sparse representations.
\newblock \emph{Machine learning}, 89\penalty0 (1-2):\penalty0 87--122, 2012.

\bibitem[Friedman(2001)]{friedman2001greedy}
Jerome~H Friedman.
\newblock Greedy function approximation: a gradient boosting machine.
\newblock \emph{Annals of statistics}, pages 1189--1232, 2001.

\bibitem[Gao and Koller(2011)]{gao2011active}
Tianshi Gao and Daphne Koller.
\newblock Active classification based on value of classifier.
\newblock In \emph{Advances in Neural Information Processing Systems}, pages
  1062--1070, 2011.

\bibitem[Gelman et~al.(2013)Gelman, Stern, Carlin, Dunson, Vehtari, and
  Rubin]{gelman2013bayesian}
Andrew Gelman, Hal~S Stern, John~B Carlin, David~B Dunson, Aki Vehtari, and
  Donald~B Rubin.
\newblock \emph{Bayesian data analysis}.
\newblock Chapman and Hall/CRC, 2013.

\bibitem[Goldberger et~al.(2000)Goldberger, Amaral, Glass, Hausdorff, Ivanov,
  Mark, Mietus, Moody, Peng, and Stanley]{goldberger2000physiobank}
Ary~L Goldberger, Luis~AN Amaral, Leon Glass, Jeffrey~M Hausdorff, Plamen~Ch
  Ivanov, Roger~G Mark, Joseph~E Mietus, George~B Moody, Chung-Kang Peng, and
  H~Eugene Stanley.
\newblock Physiobank, physiotoolkit, and physionet: components of a new
  research resource for complex physiologic signals.
\newblock \emph{circulation}, 101\penalty0 (23):\penalty0 e215--e220, 2000.

\bibitem[Gong et~al.(2019)Gong, Tschiatschek, Nowozin, Turner,
  Hern{\'a}ndez-Lobato, and Zhang]{gong2019icebreaker}
Wenbo Gong, Sebastian Tschiatschek, Sebastian Nowozin, Richard~E Turner,
  Jos{\'e}~Miguel Hern{\'a}ndez-Lobato, and Cheng Zhang.
\newblock Icebreaker: Element-wise efficient information acquisition with a
  bayesian deep latent gaussian model.
\newblock In \emph{Advances in Neural Information Processing Systems}, pages
  14791--14802, 2019.

\bibitem[Greiner et~al.(2002)Greiner, Grove, and Roth]{greiner2002learning}
Russell Greiner, Adam~J Grove, and Dan Roth.
\newblock Learning cost-sensitive active classifiers.
\newblock \emph{Artificial Intelligence}, 139\penalty0 (2):\penalty0 137--174,
  2002.

\bibitem[Hastie et~al.(2009)Hastie, Tibshirani, and
  Friedman]{hastie2009elements}
Trevor Hastie, Robert Tibshirani, and Jerome Friedman.
\newblock \emph{The elements of statistical learning: data mining, inference,
  and prediction}.
\newblock Springer Science \& Business Media, 2009.

\bibitem[Hastie et~al.(2015)Hastie, Tibshirani, and
  Wainwright]{hastie2015statistical}
Trevor Hastie, Robert Tibshirani, and Martin Wainwright.
\newblock \emph{Statistical learning with sparsity: the lasso and
  generalizations}.
\newblock CRC press, 2015.

\bibitem[He et~al.(2012)He, Eisner, and Daume]{he2012imitation}
He~He, Jason Eisner, and Hal Daume.
\newblock Imitation learning by coaching.
\newblock In \emph{Advances in Neural Information Processing Systems}, pages
  3149--3157, 2012.

\bibitem[Janisch et~al.(2017)Janisch, Pevn{\`y}, and
  Lis{\`y}]{janisch2017classification}
Jarom{\'\i}r Janisch, Tom{\'a}{\v{s}} Pevn{\`y}, and Viliam Lis{\`y}.
\newblock Classification with costly features using deep reinforcement
  learning.
\newblock \emph{arXiv preprint arXiv:1711.07364}, 2017.

\bibitem[Ji and Carin(2007)]{ji2007cost}
Shihao Ji and Lawrence Carin.
\newblock Cost-sensitive feature acquisition and classification.
\newblock \emph{Pattern Recognition}, 40\penalty0 (5):\penalty0 1474--1485,
  2007.

\bibitem[Kachuee et~al.(2018)Kachuee, Darabi, Moatamed, and
  Sarrafzadeh]{kachuee2018dynamic}
Mohammad Kachuee, Sajad Darabi, Babak Moatamed, and Majid Sarrafzadeh.
\newblock Dynamic feature acquisition using denoising autoencoders.
\newblock \emph{IEEE transactions on neural networks and learning systems},
  30\penalty0 (8):\penalty0 2252--2262, 2018.

\bibitem[Kanani and Melville(2008)]{kanani2008prediction}
Pallika Kanani and Prem Melville.
\newblock Prediction-time active feature-value acquisition for cost-effective
  customer targeting.
\newblock \emph{Advances In Neural Information Processing Systems (NIPS)},
  2008.

\bibitem[Kanao et~al.(2009)Kanao, Komori, Nakashima, Ohigashi, Kikuchi,
  Miyajima, Nakagawa, Eguchi, and Oya]{kanao2009psa}
Kent Kanao, Osamu Komori, Jun Nakashima, Takashi Ohigashi, Eiji Kikuchi, Akira
  Miyajima, Ken Nakagawa, Shinto Eguchi, and Mototsugu Oya.
\newblock Psa cut-off nomogram that avoid over-detection of prostate cancer in
  elderly men.
\newblock \emph{The Journal of Urology}, 4\penalty0 (181):\penalty0 748, 2009.

\bibitem[Kapoor and Horvitz(2009)]{kapoor2009breaking}
Ashish Kapoor and Eric Horvitz.
\newblock Breaking boundaries: Active information acquisition across learning
  and diagnosis.
\newblock In \emph{Proceedings of the 22nd International Conference on Neural
  Information Processing Systems}, pages 898--906. Curran Associates Inc.,
  2009.

\bibitem[Karayev et~al.(2013)Karayev, Fritz, and Darrell]{karayev2013dynamic}
Sergey Karayev, Mario~J Fritz, and Trevor Darrell.
\newblock Dynamic feature selection for classification on a budget.
\newblock In \emph{International Conference on Machine Learning (ICML):
  Workshop on Prediction with Sequential Models}, 2013.

\bibitem[Kusner et~al.(2014)Kusner, Chen, Zhou, Xu, Weinberger, and
  Chen]{kusner2014feature}
Matt~J Kusner, Wenlin Chen, Quan Zhou, Zhixiang~Eddie Xu, Kilian~Q Weinberger,
  and Yixin Chen.
\newblock Feature-cost sensitive learning with submodular trees of classifiers.
\newblock In \emph{AAAI}, pages 1939--1945, 2014.

\bibitem[Lakkaraju and Rudin(2017)]{lakkaraju2017learning}
Himabindu Lakkaraju and Cynthia Rudin.
\newblock Learning cost-effective and interpretable treatment regimes.
\newblock In \emph{Artificial Intelligence and Statistics}, pages 166--175,
  2017.

\bibitem[Lounici et~al.(2014)]{lounici2014high}
Karim Lounici et~al.
\newblock High-dimensional covariance matrix estimation with missing
  observations.
\newblock \emph{Bernoulli}, 20\penalty0 (3):\penalty0 1029--1058, 2014.

\bibitem[Ma et~al.(2019)Ma, Tschiatschek, Palla, Hernandez-Lobato, Nowozin, and
  Zhang]{ma2019eddi}
Chao Ma, Sebastian Tschiatschek, Konstantina Palla, Jose~Miguel
  Hernandez-Lobato, Sebastian Nowozin, and Cheng Zhang.
\newblock Eddi: Efficient dynamic discovery of high-value information with
  partial vae.
\newblock In \emph{International Conference on Machine Learning}, pages
  4234--4243, 2019.

\bibitem[Nan and Saligrama(2017)]{nan2017adaptive}
Feng Nan and Venkatesh Saligrama.
\newblock Adaptive classification for prediction under a budget.
\newblock In \emph{Advances in Neural Information Processing Systems}, pages
  4730--4740, 2017.

\bibitem[Nan et~al.(2014)Nan, Wang, Trapeznikov, and Saligrama]{nan2014fast}
Feng Nan, Joseph Wang, Kirill Trapeznikov, and Venkatesh Saligrama.
\newblock Fast margin-based cost-sensitive classification.
\newblock In \emph{Acoustics, Speech and Signal Processing (ICASSP), 2014 IEEE
  International Conference on}, pages 2952--2956. IEEE, 2014.

\bibitem[Nan et~al.(2015)Nan, Wang, and Saligrama]{nan2015feature}
Feng Nan, Joseph Wang, and Venkatesh Saligrama.
\newblock Feature-budgeted random forest.
\newblock In \emph{International Conference on Machine Learning}, pages
  1983--1991, 2015.

\bibitem[Nan et~al.(2016)Nan, Wang, and Saligrama]{nan2016pruning}
Feng Nan, Joseph Wang, and Venkatesh Saligrama.
\newblock Pruning random forests for prediction on a budget.
\newblock In \emph{Advances in neural information processing systems}, pages
  2334--2342, 2016.

\bibitem[O'Hara et~al.(2009)O'Hara, Sillanp{\"a}{\"a}, et~al.]{o2009review}
Robert~B O'Hara, Mikko~J Sillanp{\"a}{\"a}, et~al.
\newblock A review of bayesian variable selection methods: what, how and which.
\newblock \emph{Bayesian analysis}, 4\penalty0 (1):\penalty0 85--117, 2009.

\bibitem[Peter et~al.(2017)Peter, Diego, Hamprecht, and Nadler]{peter2017cost}
Sven Peter, Ferran Diego, Fred~A Hamprecht, and Boaz Nadler.
\newblock Cost efficient gradient boosting.
\newblock In \emph{Advances in Neural Information Processing Systems}, pages
  1550--1560, 2017.

\bibitem[Rasmussen and Williams(2006)]{rasmussen2006gaussian}
Carl~Edward Rasmussen and Christopher~KI Williams.
\newblock Gaussian processes for machine learning.
\newblock \emph{MIT Press}, 2006.

\bibitem[Russell and Norvig(2003)]{russell2003artificial}
Stuart Russell and Peter Norvig.
\newblock Artificial intelligence: A modern approach.
\newblock 2003.

\bibitem[Sheng and Ling(2006)]{sheng2006feature}
Victor~S Sheng and Charles~X Ling.
\newblock Feature value acquisition in testing: a sequential batch test
  algorithm.
\newblock In \emph{Proceedings of the 23rd international conference on Machine
  learning}, pages 809--816. ACM, 2006.

\bibitem[Shim et~al.(2018)Shim, Hwang, and Yang]{shim2018joint}
Hajin Shim, Sung~Ju Hwang, and Eunho Yang.
\newblock Joint active feature acquisition and classification with
  variable-size set encoding.
\newblock In \emph{Advances in Neural Information Processing Systems}, pages
  1368--1378, 2018.

\bibitem[Tibshirani(1996)]{tibshirani1996regression}
Robert Tibshirani.
\newblock Regression shrinkage and selection via the lasso.
\newblock \emph{Journal of the Royal Statistical Society. Series B
  (Methodological)}, pages 267--288, 1996.

\bibitem[Trapeznikov and Saligrama(2013)]{trapeznikov2013supervised}
Kirill Trapeznikov and Venkatesh Saligrama.
\newblock Supervised sequential classification under budget constraints.
\newblock In \emph{Artificial Intelligence and Statistics}, pages 581--589,
  2013.

\bibitem[Turney(1994)]{turney1994cost}
Peter~D Turney.
\newblock Cost-sensitive classification: Empirical evaluation of a hybrid
  genetic decision tree induction algorithm.
\newblock \emph{Journal of artificial intelligence research}, 2:\penalty0
  369--409, 1994.

\bibitem[Wang et~al.(2014{\natexlab{a}})Wang, Bolukbasi, Trapeznikov, and
  Saligrama]{wang2014model}
Joseph Wang, Tolga Bolukbasi, Kirill Trapeznikov, and Venkatesh Saligrama.
\newblock Model selection by linear programming.
\newblock In \emph{European Conference on Computer Vision}, pages 647--662.
  Springer, 2014{\natexlab{a}}.

\bibitem[Wang et~al.(2014{\natexlab{b}})Wang, Trapeznikov, and
  Saligrama]{wang2014lp}
Joseph Wang, Kirill Trapeznikov, and Venkatesh Saligrama.
\newblock An lp for sequential learning under budgets.
\newblock In \emph{Artificial Intelligence and Statistics}, pages 987--995,
  2014{\natexlab{b}}.

\bibitem[Wang et~al.(2015)Wang, Trapeznikov, and Saligrama]{wang2015efficient}
Joseph Wang, Kirill Trapeznikov, and Venkatesh Saligrama.
\newblock Efficient learning by directed acyclic graph for resource constrained
  prediction.
\newblock In \emph{Advances in Neural Information Processing Systems}, pages
  2152--2160, 2015.

\bibitem[Xu et~al.(2012)Xu, Weinberger, and Chapelle]{xu2012greedy}
Zhixiang Xu, Kilian~Q Weinberger, and Olivier Chapelle.
\newblock The greedy miser: learning under test-time budgets.
\newblock In \emph{Proceedings of the 29th International Coference on
  International Conference on Machine Learning}, pages 1299--1306. Omnipress,
  2012.

\bibitem[Xu et~al.(2013)Xu, Kusner, Weinberger, and Chen]{xu2013cost}
Zhixiang Xu, Matt Kusner, Kilian Weinberger, and Minmin Chen.
\newblock Cost-sensitive tree of classifiers.
\newblock In \emph{International Conference on Machine Learning}, pages
  133--141, 2013.

\bibitem[Zubek et~al.(2004)Zubek, Dietterich, et~al.]{zubek2004pruning}
Valentina~Bayer Zubek, Thomas~Glen Dietterich, et~al.
\newblock Pruning improves heuristic search for cost-sensitive learning.
\newblock 2004.

\end{thebibliography}

\end{document}